\documentclass[a4paper, 11pt, reqno]{amsart}
\usepackage{amsmath,amssymb,amsfonts}
\usepackage{algorithmic}
\usepackage{graphicx}
\usepackage{textcomp}
\usepackage{xcolor}
\usepackage[margin=2.4cm]{geometry}
\usepackage{microtype}
\usepackage{subcaption}
\usepackage{tikz}
\usepackage[foot]{amsaddr}

\newcommand{\Real}{\mathbb{R}}
\newcommand{\Natural}{\mathbb{N}}

\newcommand{\new}{\operatorname{new}}

\newcommand{\bfc}{c}

\newcommand{\bfm}{\boldsymbol{m}}

\newcommand{\bfv}{v}

\newcommand{\bfx}{x}
\newcommand{\bfy}{y}
\newcommand{\bfz}{\boldsymbol{z}}

\newcommand{\triplebar}{|\kern-0.1em|\kern-0.1em|}
\newcommand{\floor}[1]{\left\lfloor#1\right\rfloor}

\newcommand{\suchthat}{:}

\newtheorem{assume}{Assumption}
\newtheorem{problem}{Problem}
\newtheorem{thm}{Theorem}

\newtheorem{definition}{Definition}

\title[Towards a mathematical understanding of learning from few examples]{Towards a mathematical understanding of learning from few examples with nonlinear feature maps}

\author{Oliver J. Sutton}
\address[O. J. Sutton]{King's College, London, UK}
\email[O. J. Sutton]{oliver.sutton@kcl.ac.uk}

\author{Alexander N. Gorban}
\address[A. N. Gorban]{University of Leicester, Leicester, UK}
\email{a.n.gorban@leicester.ac.uk}

\author{Ivan Y. Tyukin}
\address[I. Y. Tyukin]{King's College, London, UK, \emph{and} University of Leicester, Leicester, UK}
\email{ivan.tyukin@kcl.ac.uk}

\begin{document}

\begin{abstract}
We consider the problem of data classification where the training set consists of just a few data points.  
We explore this phenomenon mathematically and reveal key relationships between the geometry of an AI model's feature space, the structure of the underlying data distributions, and the model's generalisation capabilities. 
The main thrust of our analysis is to reveal the influence on the model's generalisation capabilities of nonlinear feature transformations mapping the original data into high, and possibly infinite, dimensional spaces. 
\end{abstract}

\maketitle

\section{Introduction}
The last decade has seen significant progress in the application of Artificial Intelligence (AI) and Machine Learning tools to a host of practically relevant tasks.  
The availability of data, coupled with advances in computing, have led to the emergence of capable and efficient models featuring millions of trainable parameters \cite{sandler2018mobilenetv2}, \cite{tan2021efficientnetv2}.

According to classical statistical learning theory (see e.g. \cite{Bartlett1999}, Theorem 5.2), for any binary $\{0,1\}$-valued learning machine with a finite Vapnik-Chervonenkis (VC) dimension $d$, any distribution-agnostic learning algorithm, any $\epsilon >0$, and $0<\delta<1/64$, the size $m(\epsilon,\delta)$ of the training set required to ensure that, with probability $1-\delta$, the accuracy of the trained system is at most $\epsilon$-away from the best accuracy possible for this machine must satisfy
\begin{equation}\label{eq:low_bound_1}
m(\epsilon,\delta)\geq \frac{d}{320 \epsilon^2}.
\end{equation}
At the same time, according to \cite{Bartlett1999}, Theorem 8.9, the VC dimension $d$ for a class neural networks with $L$ layers, $W$ parameters, a single threshold output, and activation functions $f:\Real\rightarrow \Real$, $\lim_{s\rightarrow -\infty}f(s)=0$, $\lim_{s\rightarrow \infty}f(s)=1$ that are differentiable at some $s_0$ with $f'(s_0)\neq 0$ is known to be bounded from below by\footnote{Bound (\ref{eq:low_bound_2}) can be straightforwardly extended to networks with ReLU activation functions $f(s)=\max\{0,s\}$ by noticing that the difference of ReLU functions $f(s+b)-f(s)$, $b>0$ satisfies the assumptions of Theorem 8.9 from \cite{Bartlett1999}.}
\begin{equation}\label{eq:low_bound_2}
d\geq \floor{\frac{L}{2}}\floor{\frac{W}{2}},
\end{equation}
as long as $W\geq 10L - 14$. These inequalities suggest that accurate distribution-agnostic learning in large-scale neural networks with millions of trainable parameters may require millions of training samples -- apparently precluding the possibility of successfully learning from few examples.

Intriguingly, despite this, mounting empirical evidence points to instances when large-scale neural network models perform successfully in tasks in which the volumes of available training data do not conform to the worst-case requirements of classical Vapnik-Chervonenkis theory \cite{vapnik1999overview} or other similar combinatorial bounds. A well-known example of such a task is the classification of handwritten digits using the  MNIST digits dataset \cite{lecun2010mnist}. This dataset, being relatively small in size, can be learned remarkably well by modern large-scale deep neural networks.  This property is fascinating in its own right, especially in view of the experiments presented in \cite{zhang2016understanding}, \cite{zhang2021understanding} demonstrating that large-scale deep neural networks with identical architectures and training routines can both successfully generalise beyond training data and at the same time overfit. However, what is particularly striking, is that sometimes large-scale neural network models are capable of exhibiting extreme behaviour in comparison to worst-case bounds (\ref{eq:low_bound_1}), (\ref{eq:low_bound_2}):  {\it learning from just few examples} of objects from a new class.

To date, many successful few-shot learning schemes have been reported in the literature (see~\cite{Wang:2021FewShotReview} for a thorough review), and perhaps the best known examples of these are matching networks~\cite{vinyals2016matching} and prototypical networks~\cite{snell2017prototypical}. 
Despite the abundance of experimental confirmation of the practical feasibility of few-shot learning, a comprehensive theoretical justification of these learning schemes in large-scale models has been lacking. Although recent work \cite{bartlett2020benign} provided new relevant insights explaining the coexistence of both generalisation and  overfitting, it does not address the challenge of learning from low volumes of data.
Another relevant approach has been developed in~\cite{gorban2021high}, driven by a need to identify and correct errors made by modern high dimensional AI systems. Rather than retraining the underlying system, which may be prohibitively expensive and runs the risk of catastrophically forgetting previous training, the focus is on building simple auxiliary systems to correct~\cite{gorban2021high} or add functionality to existing AI systems.
It has been proven under certain assumptions that this approach is effective: with high probability the preexisting knowledge of the underlying system is retained and utilised when appropriate, while the new functionality is effectively learned.
This is possible because of the intrinsic properties of high dimensional spaces, where concentration of measure phenomena~\cite{ledoux2001concentration} imply the existence~\cite{Kainen:1993, Kainen:2020} and typicality~\cite{GorTyu:2016} of sets of mutually near-orthogonal points which are exponentially large in the dimension of the space.
This \emph{blessing of dimensionality} means that in sufficiently high dimensions, subsets of data points may be separated from one another with high probability using simple linear classifiers~\cite{gorban2017stochastic}.
This may be contrasted with the whole machinery of nonlinear learning algorithms which are typically required for learning in low dimensions~\cite{Gorban:2018}.

\begin{figure}
    \centering
    \includegraphics[width=0.8\textwidth]{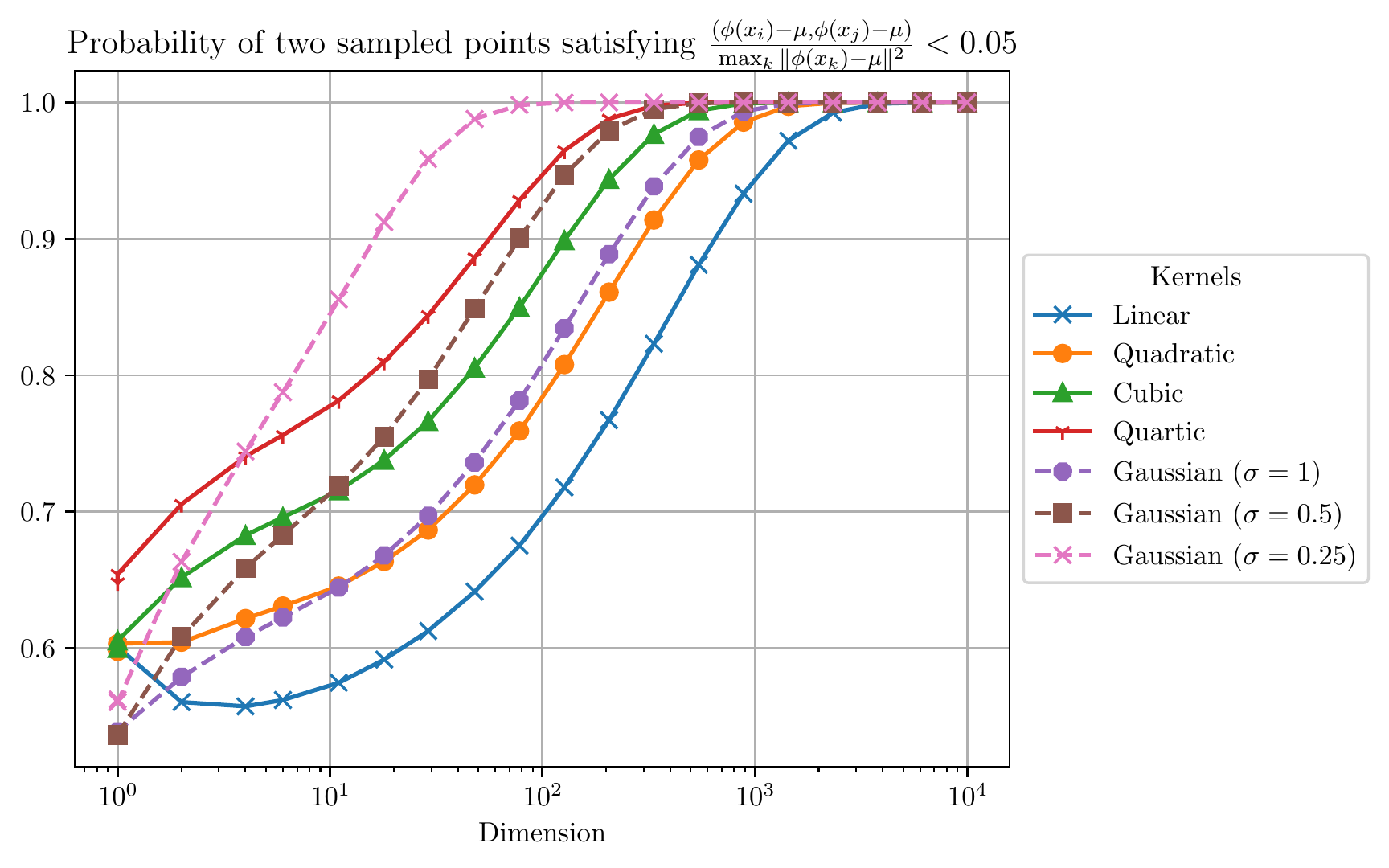}%
	\caption{Empirical demonstration of how kernels can accelerate the blessing of dimensionality by improving the degree of orthogonality of data points in the feature space associated with the kernel. In each case, $\phi$ denotes the feature map associated with the kernel, $(\cdot, \cdot)$ denotes the inner product in feature space, and $\| \cdot \|$ denotes the induced norm. The data was computed by sampling $N = 5,000$ independent uniformly distributed points $\{x_i\}_{i=1}^{N}$ from the unit ball in $\Real^d$, and $\mu = \frac{1}{N} \sum_{i=1}^{N} \phi(x_i)$ denotes the empirical mean in feature space.}
    \label{fig:kernel_separability_orthogonality}
\end{figure}

In this work we propose a new  mathematical framework building on our preliminary work~\cite{tyukin2021demystification} and aiming specifically at the analysis of the phenomenon of few-shot learning in neural networks and other models where the original input data are mapped into feature spaces via nonlinear feature maps. We ask the following natural question: can the onset of the blessing of dimensionality enabling few shot learning be accelerated by using nonlinear feature mappings to project mesodimensional datasets into higher or even infinite dimensional spaces?
This possibility is hinted at by the geometric properties of the feature spaces associated with several widely used kernels shown in Figure~\ref{fig:kernel_separability_orthogonality}, which demonstrates that normalised pairs of points sampled from these nonlinear feature spaces are very close to orthogonal.
Related work in~\cite{tyukin:kernelStochasticSeparation} has also shown that these desirable properties may be induced by nonlinear kernel methods, resulting in mapped datasets which are more highly separable.
Here we focus on understanding when this improved separability can be harnessed to facilitate \emph{learning} from few examples, and incorporate these nonlinear transformations of AI feature spaces either via an explicit feature mapping or implicitly through a kernel.
Rather than attempting to prove generalisation bounds for all possible data distributions, we investigate which properties of data distributions are relevant for facilitating few-shot learning.
This might even be a necessary refocusing of the problem in light of~\cite{bartlett2020benign, mallinar2022benign} showing that the spectrum of the data covariance matrix may hold the key to understanding benign and tempered overfitting.

Our main results (Theorems~\ref{thm:abstractFewShotLearning}, \ref{thm:lawOfHighDimension}, and ~\ref{thm:geometricProbabilities}) reveal the relationships between the data distributions and the geometry of the induced feature space enabling successful learning from few examples in a manner which will generalise to unseen data.
A key finding is that nonlinear feature mappings are beneficial for few shot learning when they ensure an appropriate combination of a high degree of orthogonality in the mapped data, a property which may be viewed as a hallmark trait of high dimensional datasets as discussed above, and localisation.
These quantities enable us to produce both upper and lower bounds (Theorem~\ref{thm:abstractFewShotLearning}) on the probabilities of successfully learning from few examples, indicating their fundamental importance for understanding few-shot learning.
Strikingly, we find that even when the regions of space occupied by each data classes are large and possibly irregular (implying that the few training samples will likely fail to capture their true extent), few shot learning may still be expected to be successful if the transformed data in each class are  also well distributed in angle around some central point.
In this setting, a very simple linear classifier in feature space has a high probability of learning and generalising.
Despite its simplicity, this linear classifier is actually very similar to those proposed and empirically studied in \cite{snell2017prototypical}. 

The paper is organised as follows. 
In Section~\ref{sec:preliminaries} we formulate the class of few-shot learning problems and the setting in which we study them. 
Section~\ref{sec:abstractResults} presents our results in an abstract setting, through which the precise relationship between the geometry of the induced feature space and the data distributions becomes evident. 
We place these abstract results into a more geometric framework in Section~\ref{sec:geometricInterpretation}, which provides a set of measurable quantities which we use to investigate the behaviour of various feature mappings in Section~\ref{sec:feature_space_estimates}.
Section \ref{sec:conclusion} concludes the paper.
A summary of some of our key notation is provided in Appendix~\ref{sec:notation}.

\section{The few-shot learning problem}\label{sec:preliminaries}

We consider the problem of \emph{few-shot learning}, i.e. learning to distinguish items of different classes based on just a few training examples, in the framework of a standard classification task.
In this framework, we assume that data points exist in the $d$-dimensional data space $\Real^d$ and there exist two sets of labels $\mathcal{L}$ and $\mathcal{L}_{new}$ such that $\mathcal{L}\cap\mathcal{L}_{new} = \emptyset$.
We further suppose that there is a previously trained classifier with classification function $F$
\begin{equation}\label{eq:classifier_general}
F: \ \Real^d \rightarrow \mathcal{L},
\end{equation}
assigning a label from the set $\mathcal{L}$ to each data point in $\Real^d$. 
The function $F$ models the existing capabilities of an AI system which was trained (possibly at great expense) before the new class labels in $\mathcal{L}_{\new}$ were available.
The key task we consider is to build a computationally cheap classifier which preserves the expertise of $F$ for data belonging to the classes in $\mathcal{L}$, yet which is also able to correctly classify data from the new classes in $\mathcal{L}_{\new}$, even from very few training examples.
For simplicity, we will typically work in the case when $\mathcal{L}_{\new}$ consists of just a single new class with label $\ell_{\new}$.

To formally state this problem, we introduce the probability distribution $P_{\mathcal{Z}}$ of legacy data-label pairs $(\bfz, \ell) \in \Real^d \times \mathcal{L}$, and the distribution $P_{\mathcal{X}}$ of new data-label pairs $(\bfx, \ell) \in \Real^d \times \mathcal{L}_{\new}$.
Associated with these are the label-agnostic marginal data distributions
\[
	P_{Z}(\bfz)=\sum_{\ell\in\mathcal{L}} P_{\mathcal{Z}}(\bfz,\ell),
	\qquad
	P_{X}(\bfx)=\sum_{\ell\in\mathcal{L}_{new}} P_{\mathcal{X}}(\bfx,\ell).
\]
Let $\mathcal{X}$ denote a finite training set of labeled data points drawn independently from $P_{\mathcal{X}}$ with the new class labels to be learned:
\[
	\mathcal{X} = \{(\bfx_i, \ell_i) \suchthat \bfx_i \in\Real^d, \ \ell_i \in \mathcal{L}_{new}\}_{i=1}^{k}.
\]
In particular, we assume that $k \ll d$, which specifies our notion of learning from ``few'' samples.

Formally, the task we consider is stated as follows (cf. \cite{tyukin2021demystification}):

\begin{problem}[Few-shot learning]\label{prob:few_shot} 
	Consider a classifier $F$ defined by~\eqref{eq:classifier_general}, trained on a sample $\mathcal{Z}$ drawn from some distribution $P_{\mathcal{Z}}$. Let $\mathcal{X}$ be a new sample that is drawn from the distribution $P_{\mathcal{X}}$ and whose cardinality $|\mathcal{X}|\ll d$.  
	Let $p_e,p_n\in(0,1]$ be given positive numbers determining the quality of learning. 

	Find an algorithm $\mathcal{A}(\mathcal{X})$ producing a new classification map 
	\[
	F_{\new}: \Real^d \rightarrow \mathcal{L}\cup \mathcal{L}_{\new},
	\]
	such that examples of class $\mathcal{L}_{\new}$ are correctly learned with probability at least $p_n$, i.e.
	\begin{equation}\label{eq:learining_from_few_1}
	P\big(F_{\new}(\bfx) \in \mathcal{L}_{\new} \big) \geq p_n,
	\end{equation}
	for $\bfx$ drawn from $P_{X}$, while $F_{\new}$ remembers the previous classifier $F$ elsewhere with probability at least $p_e$, i.e.
	\begin{equation}\label{eq:learining_from_few_2}
	P\big(F_{\new}(\bfx) = F(\bfx)\big)\geq p_e,
	\end{equation}
	for $\bfx$ drawn from the distribution $P_{Z}$.
\end{problem}

To study algorithms for tackling Problem~\ref{prob:few_shot}, we introduce a \emph{feature map} $\phi : \Real^d \rightarrow \mathbb{H}$
mapping data vectors into a Hilbert space $\mathbb{H}$ which may be either finite or infinite dimensional. 
The map $\phi$ could represent the transformation of the input data into the latent space of a deep neural network, or other relevant data transformations emerging e.g. through the application of kernel tricks or manual feature engineering.
The map $\phi$, in turn, induces a \emph{kernel}
$\kappa: \ \Real^d \times \Real^d \rightarrow \Real$,
given by
$
\kappa(\bfx,\bfy) = (\phi(\bfx),\phi(\bfy)).
$
Examples of feature maps $\phi$ include the identity map $\phi(\bfx)=\bfx$, and those associated with standard polynomial ($\kappa(\bfx,\bfy)=(\bfx \cdot \bfy + 1)^k$, $k=1,2,\dots$), or Gaussian ($\kappa(\bfx,\bfy)=\exp(-\frac{1}{2\sigma^2}\|\bfx-\bfy\|^2)$, $\sigma>0$) 
kernels, discussed in Section~\ref{sec:feature_space_estimates}.

Theorem~\ref{thm:abstractFewShotLearning} presents a simple algorithm for solving the few-shot learning problem, based on a linear classifier in feature space, and provides estimates for the probability of the success of this algorithm.
The probability estimates are explicit in the sample size $k$ and provide a guide to the design of feature spaces in which learning may successfully occur.

\section{Abstract theory of a class of few-shot learning approaches}\label{sec:abstractResults}
We consder the problem of few-shot learning in an abstract setting initially.
To concisely state our results, we introduce various probability functions in Definition~\ref{def:probabilityFunctions} which measure the prevalence of specific quantities of interest.
The algorithms we propose do not require the evaluation of these functions, which are not known in practice, although they appear in the estimates for the probabilities of success.
As such, their magnitudes provide information on the situations in which successful learning can be expected to occur.
In Sections~\ref{sec:geometricInterpretation} and~\ref{sec:feature_space_estimates} we provide concrete estimates of these functions for various common feature maps.

Throughout this section, fix $k$ as a positive integer, and let $x, y \in \Real^d$ and the set $\{x_i\}_{i=1}^{k} \subset \Real^d$ all be independently sampled from $P_{X}$, and let $z \in \Real^d$ be an observation from $P_{Z}$ which is independent of $x, y$ and $\{x_i\}_{i=1}^{k}$.
\begin{definition}[Probability functions]\label{def:probabilityFunctions}
	Let $c_{X}$ and $c_{Z}$ be arbitrary but fixed points in the feature space $\mathbb{H}$.
	We define the following shorthand notations for probabilities:
	\begin{itemize}
		\item Let $p : \Real_{\geq 0} \to [0, 1]$ denote the \emph{projection probability function}, given by
		\begin{align*}
			p(\delta) = P(x, y \sim P_{X} : (\phi(x) - c_{X}, \phi(y) - c_{X}) \leq \delta).
		\end{align*}
		\item Let $\lambda_{X}, \lambda_{X} : \Real \to [0, 1]$ denote the \emph{localisation probability functions} for $P_{X}$ and $P_{Z}$, given by
		\begin{align*}
			\lambda_{X}(r) = P(x \sim P_{X} : \|\phi(x) - c_{X}\| \leq r),
			\quad \text{and} \quad
			\lambda_{Z}(r) = P(z \sim P_{Z} : \|\phi(z) - c_{Z}\| \leq r).
		\end{align*}
		\item Let $s_{X}, s_{Z} : \Real \to [0, 1]$ denote the \emph{class separation probability functions} for $P_{X}$ and $P_{Z}$, where
		\begin{align*}
			s_{X}(\delta) = P(x \sim P_{X} : (\phi(x) - c_{X}, c_{Z} - c_{X}) \leq \delta),
		\end{align*}
		and
		\begin{align*}
			s_{Z}(\delta) = P(z \sim P_{Z} : (\phi(z) - c_{Z}, c_{X} - c_{Z}) \leq \delta).
		\end{align*}
	\end{itemize}
	Although these probabilities clearly depend on the choice of $\phi$ and the points $c_{X}$ and $c_{Z}$, we omit this from the notation for brevity.
\end{definition}

\subsection{Solution to the few-shot learning problem}

Our main result, Theorem~\ref{thm:abstractFewShotLearning} provides bounds on the probability of successfully learning from few examples in this abstract setting.
This may be viewed as an \emph{a priori} bound, in the sense that it shows the conditions under which successful few-shot learning may be expected, even though the bounds themselves may not be evaluated accurately from small data samples.
In particular, success may be expected when the probabilities $\lambda_{X}, \lambda_{Z}, s_{X}, s_{Z}$ and $p$ are sufficiently large, which geometrically corresponds to the case when 
\begin{enumerate}
\item[(i)] the two classes are well separated
\item[(ii)] the points in each class are tighly clustered together and
\item[(iii)] the points from the new class are well spread in angle around the centre point $c_{X}$.
\end{enumerate}
Condition (i) may be thought of as measuring the well-posedness of the problem, while conditions (ii) and (iii) ensure that the centre $c_{X}$ can be well estimated by the empirical mean $\mu$ of just the few available samples (see Theorem~\ref{thm:lawOfHighDimension}).

The formulation of the estimate incorporates various trade-off parameters which arise in the analysis, allowing some terms to ease the burden of others.
The results take the form of suprema over these parameters, which implies that the estimate is valid for any choice of values of these parameters, although some choices may provide a more favourable estimate than others.
The role of these parameters may be described as follows:
\begin{itemize}
	\item $\theta \in \Real$: Threshold parameter for the classifier. 
		This is a user-selected parameter which defines the separating hyperplane used as the classifier. 
		When $\theta = 0$ the separating hyperplane passes through the empirical mean of the new class data points. 
		Selecting a negative value for $\theta$ moves the hyperplane towards the old classes.
	\item $\xi, \eta$ as defined below, depending on the other parameters. 
		These control the required degree of class separation. 
		Ideally we want to select the other parameters (including $\theta$) so that both of these are large and positive, to provide a high probability that the projections of the two datasets onto the line joining $c_{Z}$ and $c_{X}$ are less than $\eta$ (for points drawn from $P_{X}$) and $\xi$ (for points drawn from $P_{Z}$), respectively.
	\item $a \geq 0$: Localisation of empirical mean of new points to new class centre. 
		For tight bounds we want a high probability that $\|\mu - c_{X}\| \leq a$.
		Consequently, we want to pick this parameter to be as big as possible to ensure this probability is high, although doing so restricts the possible range of values for $\theta$.
	\item $b, \beta \geq 0$: Localisation of points from each class about their centres. 
		For tight bounds we want a high probability that $\|\phi(x) - c_{X}\| \leq b$ for points $x$ drawn from $P_{X}$, and that $\|\phi(z) - c_{Z}\| \leq \beta$ for points $z$ drawn from $P_{Z}$.
		We therefore want to pick this parameter to be as big as possible, to allow points to be far from the centres, although doing so will restrict the range of values permissable for $\theta$.
	\item $\gamma, \epsilon > 0$: Tradeoff parameters. 
		Increasing $\gamma$ allows the requirements on the mean convergence to be relaxed when there is a large distance between the centres. 
		Increasing $\epsilon$ allows the points to be more distant from the centre of each class when the empirical mean may be expected to be close to the centre of the new class (e.g. when the points are very symmetrically distributed around $c_{X}$ in angle).
\end{itemize}

The term $P(\|\mu - c_{X}\| \leq a)$ appearing in the bound is treated separately in Theorem~\ref{thm:lawOfHighDimension}, since there are in principle many different bounds available for such a term depending on the nature and knowledge of the distribution $P_{X}$.
A complete bound may therefore be obtained by combining the results of Theorems~\ref{thm:abstractFewShotLearning} and~\ref{thm:lawOfHighDimension}.

\usetikzlibrary{arrows}
\usetikzlibrary{calc}
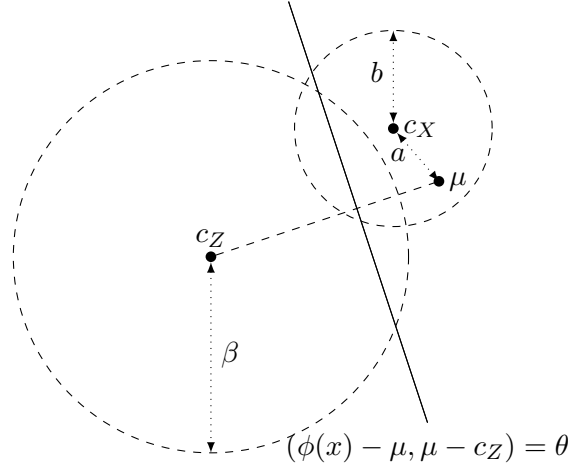
\begin{figure}
	\centering
	\begin{tikzpicture}[>=triangle 45]
		\tikzset{shorten <= 2pt}
		\tikzset{>=latex}
		\coordinate (cX) at (0, 0);
		\coordinate (mu) at (3, 1);
		\coordinate (cY) at ($(mu) + (-0.6, 0.7)$);
		\coordinate (cYmuMid) at ($(mu) + (-0.3, 0.35)$);
		\coordinate (planeStart) at (1.0, 3.45);
		\coordinate (planeEnd) at (2.85, -2.2);
		\coordinate (rXEnd) at ($(cX) + (0, -2.6)$);
		\coordinate (rXMid) at ($(cX) + (0, -1.3)$);
		\coordinate (rYEnd) at ($(cY) + (0, 1.3)$);
		\coordinate (rYMid) at ($(cY) + (0, 0.75)$);
		\fill (mu) circle [fill, radius=2pt, anchor=west] node[anchor=west] {$\mu$};
		\fill (cX) circle [fill, radius=2pt, anchor=west] node[anchor=south] {$c_{Z}$};
		\fill (cY) circle [fill, radius=2pt, anchor=west] node[anchor=west] {$c_{X}$};
		\draw[dashed] (cX) circle [fill, radius=2.6, anchor=west];
		\draw[dashed] (cY) circle [fill, radius=1.3, anchor=west];
		\draw[dashed] (cX) -- (mu);
		\draw (planeStart) -- (planeEnd);
		\draw (planeEnd) node[anchor=north] {$(\phi(x) - \mu, \mu - c_{Z}) = \theta$};
		\draw (planeStart) -- (planeEnd);
		\draw[<->, dotted] (cX) -- (rXEnd);
		\draw (rXMid) node[anchor=west] {$\beta$};
		\draw[<->, dotted] (cY) -- (rYEnd);
		\draw (rYMid) node[anchor=east] {$b$};
		\draw[<->, dotted] (cY) -- (mu);
		\draw (cYmuMid) node[anchor=east] {$a$};
	\end{tikzpicture}
	\caption{Illustration of the classifier defined in Theorem~\ref{thm:abstractFewShotLearning} and the interpretation of the parameters appearing the estimate.}
\end{figure}

\begin{thm}[Few shot learning]\label{thm:abstractFewShotLearning}
	Suppose that we are in the setting of the few-shot learning problem specified in Problem~\ref{prob:few_shot}, and let $\{x_{i}\}_{i=1}^{k} \subset \Real^d$ be independent training samples from the new data distribution $P_{X}$.
	For $\theta \in \Real$, construct the classifier
	\begin{equation}\label{eq:linearClassifier}
		F_{\new}(x) = 
		\begin{cases}
			\ell_{\new} &\text{ for } (\phi(x) - \mu, \mu - c_{Z}) \geq \theta,
			\\
			F(x) &\text{ otherwise},
		\end{cases}
	\end{equation}
	where $\mu = \frac{1}{k} \sum_{i=1}^{k} \phi(x_i)$ is the mean of the new class training samples in feature space.

	Then, with respect to samples $x$ and $\{x_{i}\}_{i=1}^{k}$ drawn independently from the distribution $P_{X}$, the probability $P(F_{\new}(x) = \ell_{\new})$ that this classifier has correctly learned and will generalise well to the new class is at least
	\begin{equation}\label{eq:learningNewClasses}
		P(F_{\new}(x) = \ell_{\new}) 
		\geq 
		\sup_{a, b, \gamma, \epsilon > 0}
		P(\|\mu - c_{X}\| \leq a)
		\{\lambda_{X}(b) + s_{X}(\eta) - 1 \}_{+},
	\end{equation}
	where $\eta = -\theta - \frac{1}{2\gamma}\|c_{X} - c_{Z}\|^2 - \frac{\epsilon + \gamma + 2}{2} a^2 - \frac{b^2}{2\epsilon}$, and this estimate is symmetric in the sense that the same terms provide an upper bound on the probability, i.e.
	\begin{equation}\label{eq:learningNewClasses:upper}
		P(F_{\new}(x) = \ell_{\new}) 
		\leq 
		1 - 
		\sup_{a, b, \gamma, \epsilon > 0}
		(1 - P(\|\mu - c_{X}\| \leq a))
		\{ 1 - \lambda_{X}(b) - s_{X}(\eta) \}_{+}.
	\end{equation}

	Moreover, with respect to samples $z$ drawn from $P_{Z}$ and $\{x_{i}\}_{i=1}^{k}$ drawn independently from $P_{X}$,
	the probability $P(F_{\new}(z) = F(z))$ that the classifier can correctly distinguish the original classes and so will retain its previous learning is at least
	\begin{equation}\label{eq:rememberingOldLearning}
		P(F_{\new}(z) = F(z)) 
		\geq 
		\sup_{a, \beta, \gamma, \epsilon > 0}
		P(\|\mu - c_{X}\| \leq a)
		\{\lambda_{Z}(\beta) + s_{Z}(\xi) - 1\}_{+},
	\end{equation}
	where 
	$\xi = \theta + \Big( 1 - \frac{1}{\gamma} \Big) \|c_{X} - c_{Z}\|^2 - \frac{\epsilon + \gamma - 2}{2} a^2 - \frac{\beta^2}{2\epsilon}$,
	and this estimate is  symmetric in the sense that
	\begin{equation}\label{eq:rememberingOldLearning:upper}
		P(F_{\new}(z) = F(z)) 
		\leq 
		1 -
		\sup_{a, \beta, \gamma, \epsilon > 0}
		(1 - P(\|\mu - c_{X}\| \leq a))
		\{1 - \lambda_{Z}(\beta) - s_{Z}(\xi)\}_{+}.
	\end{equation}
\end{thm}
\begin{proof}
	Let $x \sim P_{X}$ be independent of $\{x_i\}_{i=1}^{k}$, and therefore also of $\mu$.
	For brevity, let $d = c_{X} - c_{Z}$ denote the distance between the centres.
	Expanding, we find that the condition $(\phi(x) - \mu, \mu - c_{Z}) \geq \theta$ is equivalent to
	\begin{align*}
		(\phi(x) - c_{X}, d) + (\phi(x) - c_{X}, \mu - c_{X}) - (\mu - c_{X}, d) - \|\mu - c_{X}\|^2 \geq \theta.
	\end{align*}
	Letting $T = (\phi(x) - c_{X}, \mu - c_{X}) - (\mu - c_{X}, d) - \|\mu - c_{X}\|^2$, the Cauchy-Schwarz inequality implies that
	\[
		T
		\geq
		- \|\phi(x) - c_{X}\| \|\mu - c_{X}\| - \|\mu - c_{X}\| \|d\| - \|\mu - c_{X}\|^2,
	\]
	and, applying Young's inequality $ab \leq \frac{1}{2\epsilon} a^2 + \frac{\epsilon}{2}b^2$, valid for any $a, b \in \Real$ and $\epsilon > 0$, to the two product terms using arbitrary parameters $\gamma, \epsilon > 0$, it follows that
	\[
		T
		\geq 
		- \frac{1}{2\epsilon} \|\phi(x) - c_{X}\|^2 
		- \frac{1}{2\gamma} \|d\|^2 
		- \frac{2 + \epsilon + \gamma}{2} \|\mu - c_{X}\|^2.
	\]
	Consequently, if $x$ and $\{x_i\}_{i=1}^{k}$ are such that
	\[
		(\phi(x) - c_{X}, d) 
		- \frac{1}{2\epsilon} \|\phi(x) - c_{X}\|^2 
		- \frac{2 + \epsilon + \gamma}{2} \|\mu - c_{X}\|^2
		\geq
		\theta + \frac{1}{2\gamma} \|d\|^2,
	\]
	then it follows that $(\phi(x) - \mu, \mu - c_{Z}) \geq \theta$.
	Recalling that the points $x$ and $\{x_i\}_{i=1}^{k}$ are drawn independently from $P_{X}$, for any $a, b \geq 0$, we introduce the events
	\begin{align*}
		A: \|\mu - c_{X}\|^2 &\leq a^2,
		\qquad
		B: \|\phi(x) - c_{X}\|^2 \leq b^2,
		\\
		C: ( \phi(x) - c_{X}, d ) &\leq -\theta - \frac{1}{2\gamma}\|c_{Z} - c_{X}\|^2 - \frac{\epsilon + \gamma + 2}{2} a^2 - \frac{b^2}{2\epsilon}.
	\end{align*}
	Since event $A$ is independent of events $B$ and $C$, we conclude from the union bound that
	\[
		P((\phi(x) - \mu, \mu - c_{Z}) \geq \theta) \geq P(A \wedge B \wedge C) \geq P(A)\{P(B) + P(C) - 1\}_{+}.
	\]
	The result~\eqref{eq:learningNewClasses} then follows since $s_{X}$ is precisely $P(C)$, $\lambda_{X}$ provides $P(B)$.
	Furthermore, it follows that $\operatorname{not}A \wedge \operatorname{not}B \wedge \operatorname{not}C \Rightarrow (\phi(x) - \mu, \mu - c_{Z}) \leq \theta$,
	and consequently
	\[
		P((\phi(x) - \mu, \mu - c_{Z}) \leq \theta) 
		\geq 
		P(\operatorname{not}A \wedge \operatorname{not}B \wedge \operatorname{not}C) 
		\geq 
		(1 - P(A))\{1 - P(B) - P(C)\}_{+},
	\]
	implying that
	$
		P((\phi(x) - \mu, \mu - c_{Z}) \geq \theta)
		\leq 
		1 - (1 - P(A))\{1 - P(B) - P(C)\}_{+}.
	$
	Combined with the arguments above, this proves~\eqref{eq:learningNewClasses:upper}

	To prove the bounds on $P(F_{\new}(z) = F(z))$, let $z \sim P_{Z}$ and expand
	\[
		(\phi(z) - \mu, \mu - c_{Z})
		=
		(\phi(z) - c_{Z}, d) + (\phi(z) - c_{Z}, \mu - c_{X}) - \|\mu - c_{Z}\|^2,
	\]
	from which, arguing as before, we find that for any $\gamma, \epsilon > 0$,
	\[
		(\phi(z) - \mu, \mu - c_{Z})
		\leq
		(\phi(z) - c_{Z}, d)
		+ \frac{1}{2\epsilon} \|\phi(z) - c_{Z}\|^2
		+ \Big( \frac{\epsilon}{2} + \gamma - 1 \Big) \|\mu - c_{X}\|^2
		+ \Big( \frac{1}{\gamma} - 1 \Big) \| d \|^2.
	\]
	Consequently, we conclude that if
	\[
		(\phi(z) - c_{Z}, d)
		+ \frac{1}{2\epsilon} \|\phi(z) - c_{Z}\|^2
		+ \Big( \frac{\epsilon}{2} + \gamma - 1 \Big) \|\mu - c_{X}\|^2
		\leq 
		\theta
		+ \Big( 1 - \frac{1}{\gamma} \Big) \| d \|^2,
	\]
	then it follows that $(\phi(z) - \mu, \mu - c_{Z}) \leq \theta$.
	Let  $a, \beta \geq 0$ be arbitrary and consider the events
	\begin{align*}
		A : \|\mu - c_{X}\|^2 &\leq a^2,
		\qquad
		B : \|\phi(z) - c_{Z}\|^2 \leq \beta^2,
		\\
		C : ( \phi(z) - c_{Z}, d) &\leq \theta + \Big( 1 - \frac{1}{\gamma} \Big) \|d\|^2 - \frac{\epsilon + \gamma - 2}{2} a^2 - \frac{\beta^2}{2\epsilon},
	\end{align*}
	which, since the event $A$ is independent of the events $B$ and $C$, are such that
	\[
		P((\phi(z) - \mu, \mu - c_{Z}) \leq \theta) \geq P(A \wedge B \wedge C) \geq P(A) \{P(B) + P(C) - 1\}_{+}.
	\]
	This provides the result~\eqref{eq:rememberingOldLearning} due to the fact that $P(C)$ and $P(B)$ are given by $s_{Z}$ and $\lambda_{Z}$ respectively.
	The upper bound~\eqref{eq:rememberingOldLearning:upper} follows by arguing as for~\eqref{eq:learningNewClasses:upper}.
\end{proof}

\subsection{Convergence of the empirical mean in terms of quasi-orthogonality and locality}\label{sec:abstractResults:meanConvergence}
The final remaining piece is to estimate the distance from the empirical mean $\mu$ of the new samples $\{x_i\}_{i=1}^{k}$ to the centre $c_{X}$, providing bounds on $P(\|\mu - c_{X}\| \leq a)$.
There are many ways to derive such a bound, and we present a particularly simple argument here which is suited to the case where we have few data points.
The result shows that, despite $c_{X}$ being arbitrary, such convergence may be expected when the probabilities $p$ and $\lambda_{X}$ are sufficiently large.

These estimates once again incorporate infima over a parameter $\delta \in \Real$, implying that a valid result may be obtained by substituting any value of $\delta$.
Selecting $\delta = s^2$ in~\eqref{eq:meanConvergence:upperBound} provides a simplified estimate since in this case $ks^2 - (k-1)\delta = s^2$.
Doing so would, however, miss a key feature of this result, which is that it shows some trade-off is possible between quasi-orthogonality and localisation when estimating $c_{X}$.
Indeed, in the case when the points $\phi(x)$ are expected to be highly orthogonal in feature space, it follows that we can expect $p(t) \approx 1$ for $t > \epsilon$ for some $0 < \epsilon \ll 1$.
Selecting a small value of $\delta = \epsilon$ in the statement~\eqref{eq:meanConvergence:upperBound} therefore allows more flexibility in the localisation of the points used to calculate $\mu$, and enables the points to spread out with rate proportional to $k^{1/2}$.
This case of highly orthogonal points may be viewed as \emph{typical} of samples from high dimensional spaces and for certain well behaved feature maps $\phi$, (as explored in Section~\ref{sec:feature_space_estimates}, cf. Figure~\ref{fig:kernel_separability_orthogonality}).

Substituting the lower bound~\eqref{eq:meanConvergence:upperBound} into~\eqref{eq:learningNewClasses} and~\eqref{eq:rememberingOldLearning}, and using the upper bound~\eqref{eq:meanConvergence:sharpness} in~\eqref{eq:learningNewClasses:upper} and~\eqref{eq:rememberingOldLearning:upper} will therefore produce full explicit bounds for the few-shot learning problem in terms of our fundamental quantities.

\begin{thm}[Convergence of the empirical mean]\label{thm:lawOfHighDimension}
	Let $s > 0$, let $\{x_{i}\}_{i=1}^{k} \subset \Real^d$ be independent samples from the distribution $P_{X}$, and define $\mu = \frac{1}{k}\sum_{i=1}^{k} \phi(x_i)$.
	Then,
	\begin{align}\label{eq:meanConvergence:upperBound}
		P(\{x_i\}_{i=1}^{k} \suchthat \|\mu - c_{X}\| \leq s) 
		\geq 
			1 - 
			\inf_{\delta \in \Real}
			\big[
				k (1 - \lambda_{X}(r(s, \delta))) 
				+
				k(k-1)(1 - p(
				\delta
				))
			\big],
	\end{align}
	where $r(s, \delta) = \{ks^2 - (k-1)\delta\}_{+}^{1/2}$
	and this estimate is symmetric in the sense that
	\begin{equation}\label{eq:meanConvergence:sharpness}
		P(\{x_i\}_{i=1}^{k} \suchthat \|\mu - c_{X}\| \leq s) 
		\leq
		\inf_{\delta \in \Real}
		\big[
		k \lambda_{X}(r(s, \delta))
		+
		k(k-1) p(
			\delta
		)
		\big].
	\end{equation}
\end{thm}
\begin{proof}
	Expanding, we find that
	\begin{align*}
		\|\mu - c_{X}\|^2 
		&= 
		\frac{1}{k^2}\sum_{i = 1}^{n}\|\phi(x_i) - c_{X}\|^2 + \frac{1}{k^2}\sum_{\substack{i, j = 1, i \neq j}}^{k} (\phi(x_i) - c_{X}, \phi(x_j) - c_{X}).
	\end{align*}
	Let $\delta \in \Real$ be arbitrary, and for each $i$ and $j$, let $A_i$ be the event that $\|\phi(x_i) - c_{X}\|^2 \leq ks^2 - (k-1)\delta$ and let $B_{ij}$ be the event that $(\phi(x_i) - c_{X}, \phi(x_j) - c_{X}) \leq \delta$. 
	Then, when all these inequalities hold, the expansion above implies that
	$$
		\|\mu - c_{X}\|^2 \leq k\frac{ks^2 - (k-1)\delta}{k^2} + k(k-1)\frac{\delta}{k^2} = s^2,
	$$
	since the second sum contains $k(k-1)$ terms.
	Consequently the event $E = \bigwedge_{i = 1}^{n} A_i \wedge \bigwedge_{i, j = 1, i \neq j}^{n} B_{ij}$ implies $\|\mu - c_{X}\| \leq s$, and therefore
	\begin{align*}
		P(\|\mu - c_{X}\| \leq s) 
		\geq P(E)
		&\geq 
		1 - \sum_{i = 1}^{k} P(\operatorname{not} A_i)
		- 
		\sum_{\substack{i, j = 1, i \neq j}}^{k} 
		P(\operatorname{not} B_{ij}).
	\end{align*}
	The result~\eqref{eq:meanConvergence:upperBound} then follows from the definitions of $\lambda_{X}$ and $p$ and recalling that $\delta$ was arbitrary.
	
	To prove~\eqref{eq:meanConvergence:sharpness}, 
	note that $(\bigwedge_{i = 1}^{n} \operatorname{not} A_i \wedge \bigwedge_{i, j = 1, i \neq j}^{n} \operatorname{not} B_{ij}) \Rightarrow \|\mu - c_{X}\| > s$
	and consequently,
	\begin{align*}
		P(\|\mu - c_{X}\| > s) 
		&
		\geq
			1
			-
			\sum_{i=1}^{k} P(A_{i})
			-
			\sum_{\substack{i, j = 1, i \neq j}}^{k} 
				P(B_{ij})
	\end{align*}
	The result~\eqref{eq:meanConvergence:sharpness} then follows from the definitions of $\lambda_{X}$ and $p$ as before, since $P(\| \mu - c_{X} \| \leq s) = 1 - P(\| \mu - c_{X} \| > s)$ and $\delta$ was arbitrary.
\end{proof}

\section{Geometric interpretation of results}\label{sec:geometricInterpretation}
The abstract results in Section~\ref{sec:abstractResults} can be interpreted in a geometric setting, in which the roles of the nonlinear feature map and ambient space dimension become clear.
To describe the behaviour of feature maps in a unified manner, we introduce the function $V_{\phi} : \mathbb{H} \times \Real \to \Real_{\geq 0}$ which measures the volume of the pre-image of a ball in feature space, given by
\begin{align}\label{eq:ballPreimageVolume}
	V_{\phi}(\bfc, r) = \int_{\{ \bfx \, \suchthat \, \|\phi(\bfx) - \bfc\| \leq r \} } 1 d\bfx.
\end{align}
Similarly, we use the function $C_{\phi} : \mathbb{H} \times \mathbb{H} \times \Real \times \Real \to \Real_{\geq 0}$ to measure the volume of the pre-image of a spherical cap in feature space, defined as
\begin{align}\label{eq:capPreimageVolume}
	C_{\phi}(\bfc, \bfv, r, t) = \int_{\{ \bfx \, \suchthat \, \|\phi(\bfx) - \bfc\| \leq r \text{ and } (\phi(\bfx) - \bfc, \bfv - \bfc) \geq t\}} 1 d \bfx,
\end{align}
and we note that the spherical cap may be enveloped in a sphere, implying the trivial estimate
\begin{equation}\label{eq:capInSphereEstimate}
	C_{\phi}(\bfc, \bfv, r, t) \leq V_{\phi}(\bfc + t ( v - c), (r^2 - t^2 \|v - c\|^2)^{1/2}).
\end{equation}

We focus on data distributions satisfying the following growth bounds, which may be regarded as a generalisation of the smeared absolute continuity condition~\cite{AugmentedAI, tyukin2021demystification}.
We stress that the choice of $\bfc \in \mathbb{H}$ here is arbitrary, provided that such a ball $\mathcal{S}$ exists, and that the constant $A$ is free to depend on $\bfc$.

\begin{definition}[Distribution bounded in feature space]\label{def:bounded_distribution}
	A probability distribution 
	$P$ on $\Real^d$ is said to be \emph{bounded in feature space} with respect to the feature mapping $\phi$ if $P$ admits a density function $\rho : \Real^{d} \to \Real_{\geq 0}$, there exist a centre $\bfc \in \mathbb{H}$ and radius $r > 0$ such that $\rho$ is only supported within the set
	\[
		\operatorname{supp} \rho \subset \mathcal{S} = \{\bfx\in\Real^d \suchthat \|\phi(\bfx) - \bfc\| \leq r \}, \quad V_{d}(\mathcal{S}) = V_{\phi}(\bfc, r) > 0,
	\]
	and there exists a constant scaling $A > 0$ such that
	\begin{equation}\label{eq:def:density_growth_bound}
		\rho(\bfx) 
		\leq 
		\frac{A}
		{V_{\phi}(\bfc, r)}.
	\end{equation}
\end{definition}

The assumptions of Definition~\ref{def:bounded_distribution} require that the distribution does not have pathological concentrations. 
If such concentrations are present and identified then they can be isolated and potentially dealt with separately.
In the case of a linear feature map $\phi$, Definition~\ref{def:bounded_distribution} reduces to the condition that $\rho$ is supported in a ball and does not have singularities.

\begin{assume}\label{assume:distributionsBounded} 
	We assume that the distributions $P_{X}$ and $P_{Z}$ satisfy Definition~\ref{def:bounded_distribution} with densities $\rho_{X}$ and $\rho_{Z}$, centres $\bfc_{X}$ and $\bfc_{Z}$, radii $r_{X}$ and $r_{Z}$, and scalings $A_{X}$ and $A_{Z}$ respectively.
\end{assume}

In this context, we may write the terms used in the abstract estimates of Section~\ref{sec:abstractResults} in a geometric language.
The estimates for $\lambda_{X}, \lambda_{Z}, s_{X}$ and $s_{Z}$ are exact when $A_{X} = A_{Z} = 1$, corresponding to the case when $P_{X}$ and $P_{Z}$ are uniform distributions.

\begin{thm}[Geometric forms of probabilities]\label{thm:geometricProbabilities}
	Suppose that Assumption~\ref{assume:distributionsBounded} holds, let $V_{\phi}$ and $C_{\phi}$ be defined as in~\eqref{eq:ballPreimageVolume} and~\eqref{eq:capPreimageVolume} respectively, and let the probability functions $p$, $\lambda_{X}$, $\lambda_{Z}$, $s_{X}$ and $s_{Z}$ be as in Definition~\ref{def:probabilityFunctions}.
	Then, the projection probability function $p$ satisfies
	\begin{align}\label{eq:projectionProbabilityEstimate}
		p(\delta)
		\in
		\Big[
			1 - 
			\sup_{\bfy \in \Real^d}
			A_{X} \frac{C_{\phi}(\bfc_{X}, \phi(\bfy), r_{X}, \delta)}{V_{\phi}(\bfc_{X}, r_{X})}
			,
			\sup_{\bfy \in \Real^d}
			A_{X} \Big( 1 - \frac{C_{\phi}(\bfc_{X}, \phi(\bfy), r_{X}, \delta)}{V_{\phi}(\bfc_{X}, r_{X})} \Big)
		\Big]
		\cap [0, 1],
	\end{align}
	the localisation probability function $\lambda_{X}$ may be estimated by
	\begin{align}\label{eq:ballProbabilityEstimate}
		\lambda_{X}(r) 
		\in \Big[
			1 - A_{X} \Big(1 - \frac{V_{\phi}(\bfc_{X}, \min\{r, r_{X}\})}{V_{\phi}(\bfc_{X}, r_{X})} \Big)
			,
			A_{X} \frac{V_{\phi}(\bfc_{X}, \min\{r, r_{X}\})}{V_{\phi}(\bfc_{X}, r_{X})}
		\Big]
		\cap [0, 1],
	\end{align}
	with an analogous estimate for $\lambda_{Z}$, and the separation probability function $s_{X}$ satisfies
	\begin{align}\label{eq:classSeparationProbabilityEstimate}
		s_{X}(\delta)
		\in
		\Big[
			1 - A_{X} 
				\frac{C_{\phi}(\bfc_{X}, \bfc_{Z}, r_{X}, \delta)}{V_{\phi}(\bfc_{X}, r_{X})}
			,
			A_{X}
			\Big(
				1 - 
				\frac{C_{\phi}(\bfc_{X}, \bfc_{Z}, r_{X}, \delta)}{V_{\phi}(\bfc_{X}, r_{X})}
			\Big)
		\Big]
		\cap [0, 1],
	\end{align}
	with an analogous estimate holding for $s_{Z}$.
\end{thm}
\begin{proof}
	From the definition of $p$, we find that
	\begin{align*}
		p(\delta) &= P(x, y \sim P_{X} : (\phi(x) - c_{X}, \phi(y) - c_{X}) \leq \delta)
		= \int_{\Real^d} \int_{S(\phi(y), \delta) } \rho_{X}(\bfx) \rho_{X}(\bfy) d \bfx d \bfy.
	\end{align*}
	where $S(q, \delta) = \{ \bfx \, \suchthat \, (\phi(\bfx) - \bfc_{X}, q - \bfc_{X}) \leq \delta \}$.
	Since $\rho_{X}$ is only supported in a ball of radius $r_{X}$ in feature space, this may be further expressed as
	\begin{align*}
		p(\delta)
		&
		= 
		\int_{\{\bfy \, \suchthat \, \|\phi(\bfy) - \bfc_{X}\| \leq r_{X} \}} 
		\int_{\hat{S}(y, r, \delta)} 
		\rho_{X}(\bfx) d \bfx \, \rho_{X}(\bfy) d \bfy.
	\end{align*}
	where $\hat{S}(y, r, \delta) = \{ \bfx \, \suchthat \, \|\phi(\bfx) - \bfc_{X}\| \leq r_{X} \text{ and } (\phi(\bfx) - \bfc_{X}, \phi(\bfy) - \bfc_{X}) \leq \delta\}$.
	Applying the bound~\eqref{eq:def:density_growth_bound} on $\rho_{X}$, we further deduce that
	\begin{align*}
		p(\delta)
		&
		\leq
		\frac{A_{X}}{V_{\phi}(\bfc_{X}, r_{X})}
		\int_{\{\bfy \, \suchthat \, \|\phi(\bfy) - \bfc_{X}\| \leq r_{X} \}} 
		\int_{\hat{S}(y, r, \delta)} 
		1 d \bfx \, \rho_{X}(y) d \bfy,
	\end{align*}
	and recalling the definition of $C_{\phi}$ and the fact that $\int_{\Real^d} \rho_{X}(\bfy) d \bfy = 1$, we obtain
	$$
		p(\delta) \leq 
		\sup_{\bfy \in \Real^d}
		A_{X} \Big(1 - \frac{C_{\phi}(\bfc_{X}, \phi(\bfy), r_{X}, \delta)}{V_{\phi}(\bfc_{X}, r_{X})} \Big).
	$$
	Arguing similarly for
	$
		1 - p(\delta) 
		= P(x, y \sim P_{X} : (\phi(x) - c_{X}, \phi(y) - c_{X}) \geq \delta),
	$
	we obtain a lower bound on $p(\delta)$
	and~\eqref{eq:projectionProbabilityEstimate} therefore follows.

	To estimate $\lambda_{X}(r)$, we observe that
	\begin{align*}
		\lambda_{X}(r) &= 
		P(x \sim P_{X} : \|\phi(x) - c_{X}\| \leq r)
		= \int_{\{\bfx \, \suchthat \, \|\phi(\bfx) - \bfc_{X}\| \leq r\}} \rho_{X}(\bfx) d \bfx
		\leq
		A_{X} \frac{V_{\phi}(\bfc_{X}, \min\{r, r_{X}\})}{V_{\phi}(\bfc_{X}, r_{X})}.
	\end{align*}
	Arguing similarly for $1 - \lambda_{X}(r)$
	implies~\eqref{eq:ballProbabilityEstimate}.

	Turning to the class separation probability, and arguing as for $p(\delta)$, we have
	\begin{align*}
		s_{X}(\delta) 
		&= 
		P(x \sim P_{X} : x \in S(c_{Z}, \delta))
		=
		\int_{S(c_{Z}, \delta)} 
		\rho_{X}(\bfx) d \bfx
		\leq
		A_{X}
		\Big(
			1 -
			\frac{C_{\phi}(\bfc_{X}, \bfc_{Z}, r_{X}, \delta)}{V_{\phi}(\bfc_{X}, r_{X})}
		\Big),
	\end{align*}
	and the estimate~\eqref{eq:classSeparationProbabilityEstimate} follows by arguing similarly for $1 - s_{X}(\delta)$.
\end{proof}

Our results so far may therefore be summarised as follows:
\begin{center}
	\begin{tabular}{ c | c | c }
	 Term & For successful learning & Which is ensured by \\ 
	 \hline
	 $p(\delta)$ & $\to H(\delta)$ & $\frac{C_{\phi}(c_{X}, \phi(y), r_{X}, \delta r_{X})}{V_{\phi}(c_{X}, r_{X})} \to 1 - H(\delta)$ \\  
	 $\lambda_{X}(r)$ [sim. $\lambda_{Z}(r)$] & $\to 1$ & $\frac{V_{\phi}(c_{X}, r)}{V_{\phi}(c_{X}, r_{X})} \to 1$ \\
	 $s_{X}(\delta)$ [sim. $s_{Z}(\delta)$] & $\to H(\delta)$ & $\frac{C_{\phi}(c_{X}, c_{Z}, r_{X}, \delta)}{V_{\phi}(c_{X}, r_{X})} \to 1 - H(\delta)$

	\end{tabular}
\end{center}
where $H(\delta)$ denotes the Heaviside function.

Loosely speaking, our results in this context therefore involve quantities of the form\footnote{Note that the spherical cap volume may be estimated by enveloping the spherical cap in a sphere as in~\eqref{eq:capInSphereEstimate}.}
\[
	A_{X}\frac{V_{\phi}(\bfc, \epsilon r)}{V_{\phi}(\bfc, r)},
\]
for some $\epsilon \in (0, 1)$.
The interpretation of such a term is most clearly demonstrated when $\phi$ is simply the identity map. 
In this case $V_{\phi}(\bfc, r)$ is just the volume of a ball in $\Real^d$, and so
\[
	A_{X}\frac{V_{\phi}(\bfc, \epsilon r)}{V_{\phi}(\bfc, r)} = A_{X}\epsilon^d,
\]
and we observe exponential convergence of this quantity to 0 with respect to the data space dimension $d$.
A key question, therefore, is whether using nonlinear kernels offers any improvement over this rate,
which we investigate further in Section~\ref{sec:feature_space_estimates}.

\section{Estimates for families of feature maps}\label{sec:feature_space_estimates}
We now investigate the behaviour of the ratios appearing in Theorem~\ref{thm:geometricProbabilities}. 
In some cases we are able to perform this investigation analytically, whilst in others we turn instead to numerical simulations using the following algorithm.

\subsection{Numerical algorithm}\label{sec:numericalAlgorithm}
Defining the data domain to be a set $D \subset \Real^d$, we sample $k$ points $\{x_{i}\}_{i=1}^{k} \subset D$ from a distribution $P$ satisfying Definition~\ref{def:bounded_distribution}, and select the centre $c = \frac{1}{k}\sum_{i=1}^{k} \phi(x_{i})$ as their empirical mean, providing an approximation to the point $\hat{c} = \int_{D} \phi(x) \rho(x) dx$.
We also use this sample to estimate the minimal radius $r$ of the ball in feature space centred at $c$ such that $\operatorname{supp}(p) \subset S = \{ x \in \Real^d \suchthat \| \phi(x) - c \| \leq r \}$, as required by Definition~\ref{def:bounded_distribution}.
Taking a separate sample $\{y_{i}\}_{i=1}^{K}$ of $K$ points uniformly distributed in $\operatorname{supp}(p)$, we are able to use a simple Monte-Carlo approach to approximate the required volume ratios.
By computing 
$$
	L = | \{ y \in \{y_{i}\}_{i=1}^{K} : \| \phi(y) - c \| \leq \epsilon r \} |,
$$
we use the approximation
$
	\frac{V_{\phi}(c, \epsilon r)}{V_{\phi}(c, r)} \approx \frac{L}{K}.
$
Similarly, for $z \in D$ we define
$$
	M = | \{ y \in \{y_{i}\}_{i=1}^{K} : \| \phi(y) - c \| \leq r \text{ and } (\phi(y) - c, \phi(z) - c) \geq \delta \} |,
$$
and are therefore able to approximate
$
	\frac{C_{\phi}(c, \phi(z), r, \delta)}{V_{\phi}(c, r)} \approx \frac{M}{K}.
$

These quantities can be evaluated without explicitly knowing or evaluating the feature map or centre $c$ when $\phi$ is defined via a kernel $\kappa : \Real^d \times \Real^d \to \Real$ with $\kappa(x, y) = (\phi(x), \phi(y))$. 
This is important because it allows us to study kernels where the associated feature space is infinite dimensional.
We only need to evaluate
$$
	\|\phi(y) - c\|^2 = \kappa(y, y) - \frac{2}{k} \sum_{i=1}^{k} \kappa(y, x_i) + \frac{1}{k^2} \sum_{i=1}^{k} \sum_{j=1}^{k} \kappa(x_i, x_j),
$$
and
$$
	(\phi(y) - c, \phi(z) - c) = \kappa(y, z) - \frac{1}{k} \Big( \sum_{i=1}^{k} \kappa(y, x_i) + \sum_{i=1}^{k} \kappa(z, x_i) \Big) + \frac{1}{k^2} \sum_{i=1}^{k} \sum_{j=1}^{k} \kappa(x_i, x_j).
$$

\subsection{Polynomial kernels}
Consider the polynomial kernel given by
\[
	\kappa(\bfx, \bfy) = (b^2 + \bfx \cdot \bfy)^k
\]
for some $b \geq 0$ and integer $k \geq 1$.
In the simplest cases of $k = 1$ and $k = 2$ we are able to derive analytical estimates for $\frac{V_{\phi}(\bfc, \epsilon r)}{V_{\phi}(\bfc, r)}$, which behaves like $\epsilon^{d}$ or $\epsilon^{d/2}$ in these two cases respectively, showing the benefits of using quadratic kernels.
We then conduct a numerical investigation to show the behaviour of this ratio and the ratio
$\frac{C_{\phi}(\bfc, \bfv, \epsilon r, \delta)}{V_{\phi}(\bfc, r)}$
when using higher order kernels.

For the simplicity of the exposition, we introduce the \emph{multi-index} $\bfm \in \mathbb{Z}_{\geq 0}^d$, such that
$
	\bfx^{\bfm} = \prod_{i=1}^{d} \bfx_i^{\bfm_i},
$
with total degree given by $|\bfm| = \sum_{i=1}^{d} \bfm_i$.
We impose an (arbitrary but fixed) indexing $\{\bfm^i\}_{i=1}^{D}$ on the $D = \binom{d + k}{k}$ multi-indices $\bfm$ with $|\bfm| < k$.
This indexing is assumed to be such that for all $1 \leq i, j \leq D$ we have $\bfm^i = \bfm^j \iff i = j$ and such that 
$|\bfm^i| \leq |\bfm^j| \Rightarrow i \leq j$.

With this notation, we may express the feature map associated with the kernel $\kappa$ as the vector-valued function $\phi : \Real^d \to \Real^D$ with
\[
	[\phi(\bfx)]_{i} = \alpha(\bfm^{i}) \bfx^{\bfm^i},
\quad\text{where}\quad
	\alpha(\bfm) = \left( \binom{k}{k - |\bfm|} b^{2(k - |\bfm|)} \prod_{t=1}^{d} \binom{\sum_{j=t}^{d} \bfm_j}{\bfm_t} \right)^{1/2},
\]
and let $C \in \Real^{D \times D}$ be the diagonal coefficient matrix corresponding to $\phi$ given by
$
	C_{ii} = \alpha(\bfm^{i}).
$

We suppose that points $\bfx$ are sampled from a probability distribution $P$ satisfying Definition~\ref{def:bounded_distribution} with density $\rho(\bfx) : \Real^d \to \Real_{\geq 0}$, radius $r_P > 0$, and centre $\bfc \in \mathbb{H}$ given by the expectation
$
	\bfc = \int_{\Real^d} \phi(\bfy) \rho(\bfy) d \bfy.
$
This choice of $c$ implicitly assumes that all moments of $\rho$ up to degree $k$ are finite, and we define $M(\rho) \in \Real^D$ where
$
	[M(\rho)]_i = \int_{\Real^d} \bfy^{\bfm^i} \rho(\bfy) d \bfy.
$

We now wish to estimate $V_{\phi}(c, r)$ for $r \geq 0$, defined in~\eqref{eq:ballPreimageVolume}.
Expanding the norm and applying the definition of the kernel $\kappa$, we find that
\begin{equation}\label{eq:polynomial:generalIndicator}
	\|\phi(\bfx) - \bfc\| \leq r 
	\, \iff \, 
	(b^2 + |\bfx|^2)^k - 2 M(\rho) C^{\top} \phi(\bfx) + \| CM(\rho) \|^2 - r^2 \leq 0,
\end{equation}
and the set of such $\bfx$ is therefore always bounded since the leading order term $|\bfx|^{2k}$ has a positive coefficient, and its boundary is the level set of a polynomial function.

When $P$ is a uniform distribution over the cube $[-L, L]^d$, we may compute $M(\rho)$ as
\begin{equation}\label{eq:momentValues}
	[M(\rho)]_{\bfm} = 
	\int_{[-L, L]^d} \bfx^{\bfm} \rho(\bfx) d \bfx =
	\begin{cases}
		0 &\text{ if any component of $\bfm$ is odd},
		\\
		\frac{L^{|\bfm|}}{\prod_{i=1}^{d} (\bfm_i + 1)} &\text{ otherwise}.
	\end{cases}
\end{equation}
This implies $[M(\rho)]_{\bfm}$ is only non-zero when $|\bfm|$ is even or zero, since if $|\bfm|$ is odd then $\bfm$ must have an odd component.

\subsubsection{Linear kernels}
For a linear kernel (i.e. $k = 1$), \eqref{eq:momentValues} implies that
$\|CM(\rho)\|^2 = b^2$, and it then follows that $\|\phi(\bfx) - \bfc\| \leq r$ if and only if
$
	|\bfx| \leq r,
$
and therefore
$$
	\frac{V_{\phi}(\bfc, \epsilon r)}{V_{\phi}(\bfc, r)}
	=
	\epsilon^{d}.
$$

\subsubsection{Quadratic kernels}

Next consider the quadratic case when $k=2$.
Evaluating the entries of the moment vector using~\eqref{eq:momentValues} implies that
$
	\|C M(\rho)\|^2 = \frac{dL^4}{9} + b^4,
$
due to the definition of $C$.
In this case~\eqref{eq:polynomial:generalIndicator} may be rearragned to give
\[
	|\bfx|^4 + 2 \Big( b^2 - \frac{L^2}{3} \Big) |\bfx|^2  - 2b^2 + \frac{dL^4}{9} + 2b^4 - r^2 \leq 0,
\]
and consequently we conclude that $\|\phi(x) - \bfc\| \leq r$ if and only if
\[
	|\bfx|^2 \in \Big[ 
		\max\Big\{0, \frac{L^2}{3} - b^2 - \xi \Big\}
		,
		\frac{L^2}{3} - b^2 + \xi
	\Big],
	\text{ where }
	\xi^2 = r^2 - b^4 + 2\Big(1 - \frac{L^2}{3} \Big)b^2 - \frac{(d - 1)L^4}{9},
\]
which may take the form of either a ball or an annulus in the data space $\Real^d$.
Since we have assumed that $P$ is a uniform distribution over a cube, it follows that the dataset cannot be contained in an annulus centred at the origin, and therefore
$$
	\frac{V_{\phi}(\bfc, \epsilon r)}{V_{\phi}(\bfc, r)} 
	\leq 
	\frac{
		\Big( 
		\frac{L^2}{3} - b^2 + \Big( \epsilon^2 r^2 - b^4 + 2\Big(1 - \frac{L^2}{3} \Big)b^2 - \frac{(d - 1)L^4}{9} \Big)^{1/2}
	\Big)^{d/2}
	}{
		\Big( 
		\frac{L^2}{3} - b^2 + \Big( r^2 - b^4 + 2\Big(1 - \frac{L^2}{3} \Big)b^2 - \frac{(d - 1)L^4}{9} \Big)^{1/2}
	\Big)^{d/2}
	}.
$$
The dependence on $\epsilon$ is exposed in the simplifying case when $L = \frac{1}{\sqrt{3}}$ and $b = 1$, and by writing $r^2 = (1 + \delta)d$ for some $\delta \geq -1$, implying that
$$
	\frac{V_{\phi}(\bfc, \epsilon r)}{V_{\phi}(\bfc, r)} 
	\leq 
	\big( 
		\epsilon^2(1 + \delta^{-1}) - \delta^{-1}
	\big)^{d/4},
$$
and therefore for $\delta$ sufficiently large this bound will behave as $\epsilon^{d/2}$.
This is less restrictive than the order $\epsilon^d$ bound obtained for the linear kernel, which implies a potential advantage may be obtained by using higher order kernels.
Since the values of $\lambda_{X}(r)$ and $\lambda_{Z}(r)$ are tied to this ratio, the bound suggests that the quadratic feature map is capable of producing feature vectors which are more closely concentrated around their mean than those produced by the linear feature map.
Of course, the quality of the learning which may be expected is also related to the cap volume ratio, an analytical treatment of which is beyond the scope of this article.
Our numerical investigation (discussed in Section~\ref{sec:polynomial_numerics}) investigates both these ratios in further detail, also incorporating higher degree kernels.

\subsubsection{Numerical investigation}\label{sec:polynomial_numerics}

\begin{figure}[h!]
    \centering
	\begin{subfigure}{0.9\textwidth}
    	\includegraphics[width=\textwidth]{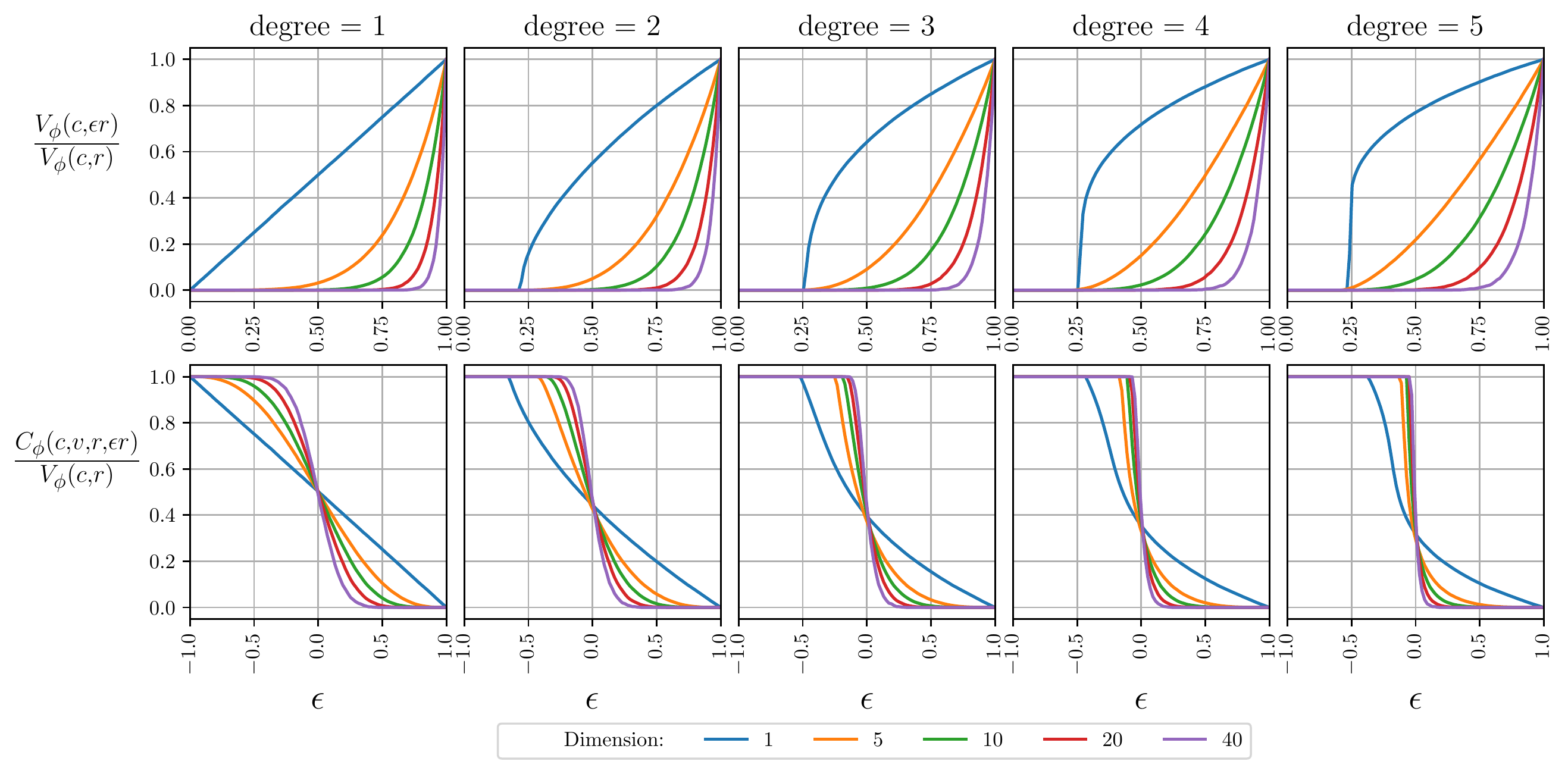}
		\caption{Effects of varying data space dimension with fixed polynomial degree.}
	\end{subfigure}

	\begin{subfigure}{0.9\textwidth}
    	\includegraphics[width=\textwidth]{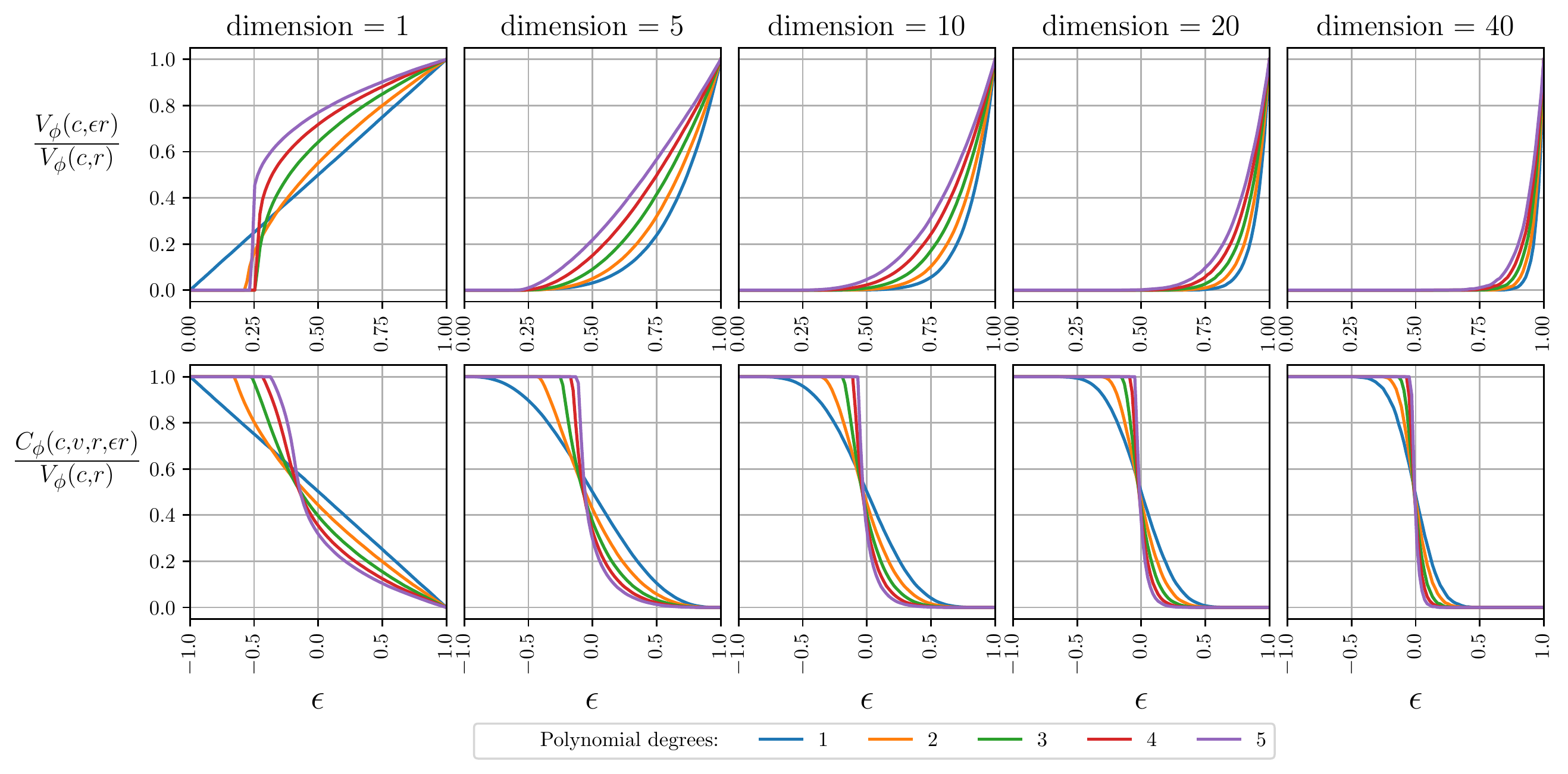}
		\caption{Effects of varying polynomial degree with fixed data space dimension.}
	\end{subfigure}
    \caption{Numerical estimates of volume ratios for polynomial kernels of varying degrees in a variety of different data space dimensions. Both figures show the same data presented differently to facilitate comparison.}
    \label{fig:polynomial:volumeRatios}
\end{figure}

The investigation is continued numerically to incorporate higher polynomial degrees and to estimate the spherical cap volume $C_{\phi}$.
The results are presented in Figure~\ref{fig:polynomial:volumeRatios}, and were computed using the algorithm described in Section~\ref{sec:numericalAlgorithm} with $k = 10^3$, $K = 10^5$, and $P$ taken as a uniform distribution in the unit ball in $\Real^d$.
The bias parameter of the kernel was taken as $1$ in all cases, and we selected $v = \phi(y)$ where $y \in \Real^d$ has  has first component 1 and is zero otherwise.

Recall that to ensure successful learning we want the ball volume ratio $\frac{V_{\phi}(c, r)}{V_{\phi}(c, r)}$, measuring how tightly clustered the data is in feature space, close to $1$.
We observe that, for a fixed data space dimension $d$, increasing the polynomial degree of the kernel does indeed drive the ball volume ratio closer to 1, implying that learning becomes easier in higher order feature spaces.
This extends the same trend shown theoretically for linear and quadratic kernels above.
On the other hand, the ratio decreases as the dimension of the data space increases for fixed polynomial degree.
This may be expected, since the volume of a ball in a high dimensional space concentrates around its surface, and the feature space dimension $D = \binom{d + k}{k}$ for a polynomial kernel of degree $k$ grows quickly with $d$.

Out theoretical results also assert that for successful learning we want the spherical cap volume ratio $\frac{C_{\phi}(c, \phi(y), r, \epsilon r)}{V_{\phi}(c, r)}$, which estimates the expected degree of quasi-orthogonality and potential for class separation, to behave like $1 - H(\epsilon)$ where $H$ is the Heaviside function.
The results show that, at least within the range of parameters and dimensions explored, this occurs as either the dimension of the data space or the degree of the polynomial kernel are increased.

To conclude this investigation, we observe that, for the range of parameters and dimensions assessed in our experiments, while increasing the data space dimension only improves $s_{X}, s_{Z}$ and $p$, increasing the kernel degree also improves $\lambda_{X}$ and $\lambda_{Z}$.

\subsection{Gaussian kernels}

Consider now the kernel with the form
\[
	\kappa(\bfx, \bfy) = \exp \Big(-\frac{1}{2\sigma}|\bfx - \bfy|^2 \Big),
\]
for some $\sigma > 0$,
which induces a feature map $\phi$ mapping $\Real^d$ onto the unit sphere in the (countably infinite dimensional) space $\ell^2$.
The behaviour of $V_{\phi}(\bfc, r)$ is quite subtle, as balls in feature space can encompass the whole of the image of $\Real^d$ for even finite $r$.
For example, consider the case when $\bfc = 0 \in \ell^2$. 
Since for $r < 1$ the ball $\|\phi(\bfx) - 0\| < r$ does not intersect with the unit sphere, while for $r \geq 1$ it contains the whole sphere, we find that
\begin{equation}\label{eq:gaussian:zeroCentred}
	V_{\phi}(0, r) =
	\begin{cases}
		0 \quad &\text{for } r < 1,
		\\
		\infty &\text{otherwise}.
	\end{cases}
\end{equation}

Alternatively, taking $\bfc = \phi(\bfy)$ for some $\bfy \in \Real^d$, 
we may compute $V_{\phi}(\phi(\bfy), r)$ as $\kappa$ implies
\[
	\| \phi(\bfx) - \phi(\bfy) \|^2 \leq r^2
	\iff
	|\bfx - \bfy|^2 \leq -2 \sigma \log \Big(1 - \frac{1}{2}r^2 \Big),
\]
for $r \leq \sqrt{2}$.
For $r > \sqrt{2}$, any points $\bfx, \bfy \in \Real^d$ satisfy $\|\phi(\bfx) - \phi(\bfy)\| < r$, implying that $\phi$ is only mapping to the intersection of the sphere with a simplex in $\ell^2$.
Thus, for any $\bfy \in \Real^d$,
\[
	V_{\phi}(\phi(\bfy), r) = 
	\begin{cases}
		\frac{\pi^{d/2}}{\Gamma(\frac{d}{2} + 1)}\big(-2 \sigma \log \big(1 - \frac{1}{2}r^2 \big) \big)^{d/2}
		\quad&
		\text{for } r^2 \leq 2,
		\\
		\infty &\text{otherwise}.
	\end{cases}
\]

\subsubsection{Numerical investigation}

\begin{figure}
    \centering
	\begin{subfigure}{0.5\textwidth}
    	\includegraphics[width=\textwidth]{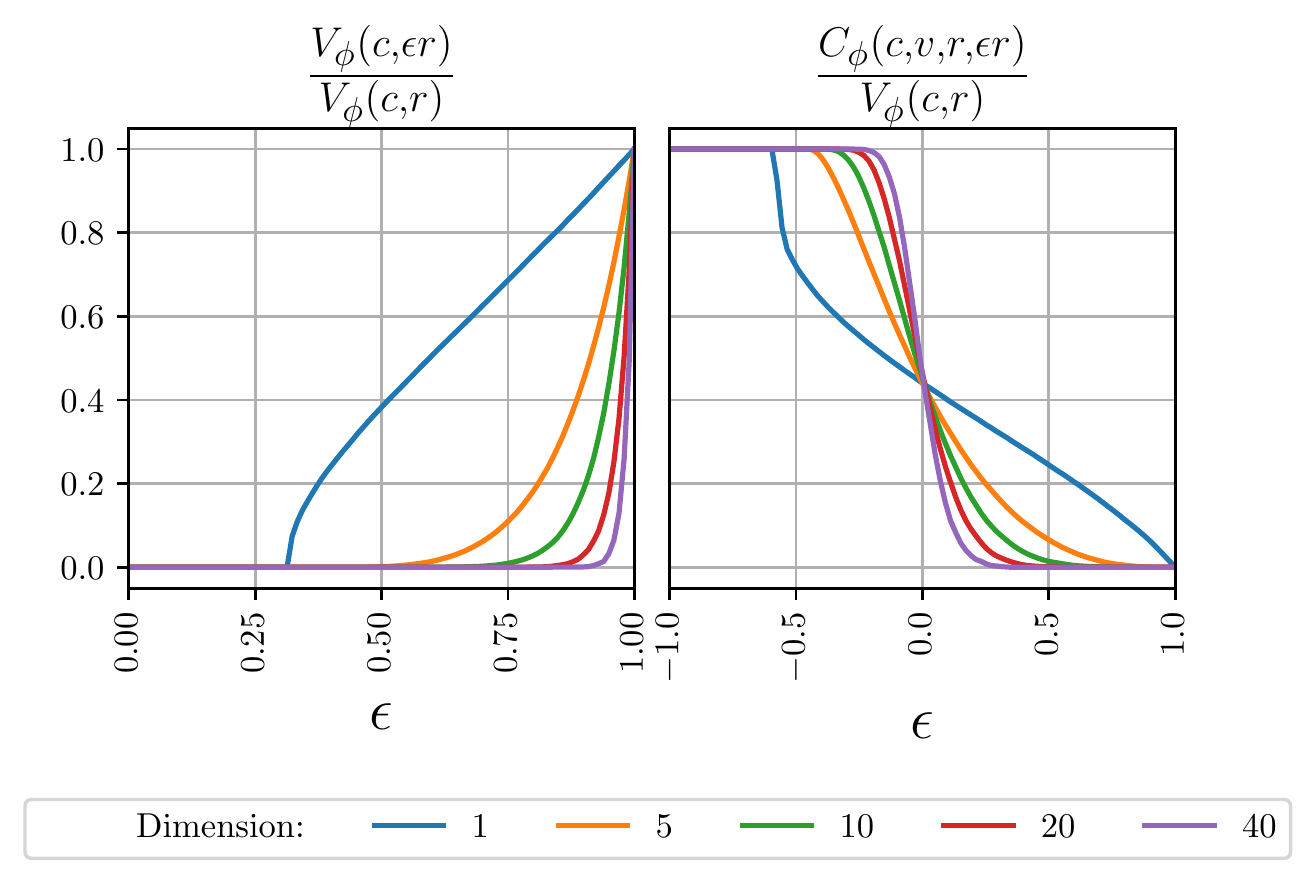}
		\caption{With $\sigma = 1$.}
	\end{subfigure}%
	\begin{subfigure}{0.5\textwidth}
    	\includegraphics[width=\textwidth]{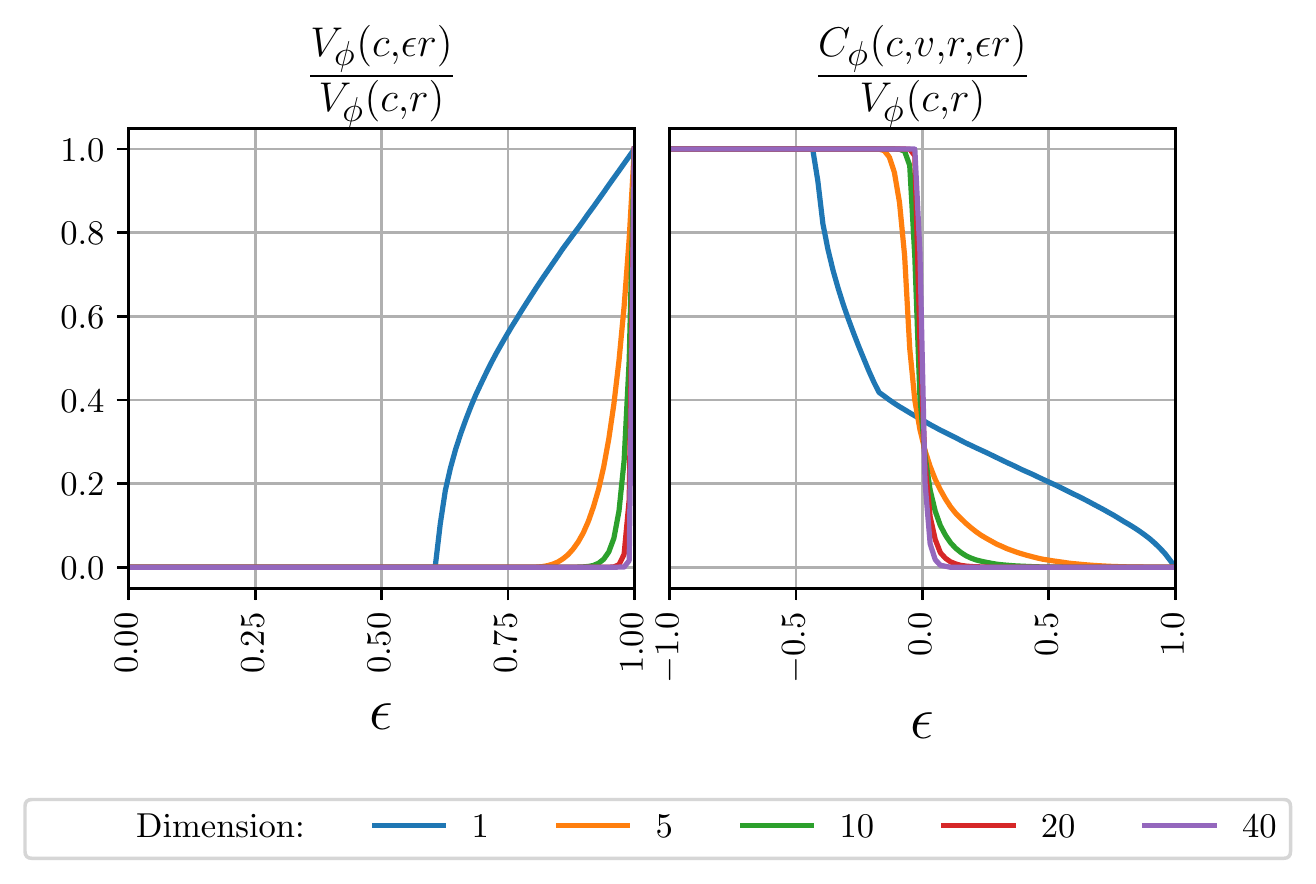}
		\caption{With $\sigma = \frac{1}{4}$.}
	\end{subfigure}
    \caption{Numerical estimates of volume ratios for the Gaussian kernel in a variety of different data space dimensions for different values of the parameter $\sigma$.}
    \label{fig:gaussian:volumeRatios}
\end{figure}

The results of investigating these quantities numerically using the algorithm described in Section~\ref{sec:numericalAlgorithm} are presented in Figure~\ref{fig:gaussian:volumeRatios}.
As before, these results were computed using $k = 10^3$ and $K = 10^5$, with $P$ taken as a uniform distribution in the unit ball in $\Real^d$.
We selected $v = \phi(y)$ where $y \in \Real^d$ has first component 1 and is zero otherwise.

The theoretical results show that for successful learning we want the spherical cap volume ratio $\frac{C_{\phi}(c, \phi(y), r, \epsilon r)}{V_{\phi}(c, r)}$, which estimates the degree of quasi-orthogonality expected in feature space, to behave like $1 - H(\epsilon)$ where $H$ is the Heaviside function.
This is precisely the behaviour we observe empirically, implying that points are close to orthogonal in feature space.
On the other hand, recall that we want the ball volume ratio $\frac{V_{\phi}(c, \epsilon r)}{V_{\phi}(c, r)}$, measuring how tightly clustered the data is in feature space, close to 1.
As in the polynomial case, this ratio decreases rapidly as the data space dimension $d$ increases, implying that the points in feature space are spread further apart as the dimension of the data space increases.

Decreasing $\sigma$ increases the degree of orthogonality of the images of points in the feature space, but it also spreads points further from the empirical mean $\mu$.
This is because
$$
	\|\mu\|^2 = \frac{1}{k^2} \sum_{i, j = 1}^{k} \exp\Big(-\frac{1}{2\sigma} |x_i - x_j|^2 \Big) 
	\to
	\begin{cases}
		\frac{1}{k} &\text{ as } \sigma \to 0,
		\\
		\frac{1}{k} + \Big( 1 - \frac{1}{k} \Big) \exp\Big(-\frac{1}{\sigma}\Big) &\text{ as } d \to \infty,
	\end{cases}
$$
since the $i \neq j$ terms will be negligible when $\sigma$ is small, and for large $d$ we expect points sampled uniformly from the unit ball to be near orthogonal, and therefore for $\|x_i - x_j\|$ to concentrate around $\sqrt{2}$.
This also explains why, for $d = 1$, we see $V_{\phi}(c, \epsilon r) = 0$ for $\epsilon < 0.25$ when $\sigma = 1$ and for $\epsilon < 0.55$ for $\sigma = \frac{1}{2}$: since $\phi$ maps the data to the surface of the unit sphere in $\ell^2$, there is a range of radii such that a sphere centred at $\mu$ does not contain any mapped data.
The implication of this phenomenon is that we may find ourselves in the situation of~\eqref{eq:gaussian:zeroCentred} as $\|\mu\|$ approaches zero, and therefore successful learning becomes less likely.

A key conclusion from this investigation is that while using Gaussian kernels with a small variance parameter $\sigma$ introduces significant orthogonality to the data in feature space, it is not clear that this is necessarily beneficial for learning as it simultaneously decreases the degree of localisation of the data.
Roughly speaking, this might imply that although the data become very easy to separate in feature space, the lack of localisation could render it difficult to actually \emph{learn} from the data as disparate points have little in common.

\subsection{Neural networks}
We now turn to look at using the feature space of a neural network as the nonlinear mapping in which we perform few shot learning, applied to the MNIST dataset of handwritten digits.
A relatively simple convolutional neural network was used for this task, the details of which are described in Table~\ref{table:NNparameters}, which as a reference was able to classify the MNIST dataset in a conventional setting with approximately 99\% test accuracy.
The Keras API for Tensorflow version 2.8~\cite{tensorflow} was used for the experiments in this section.

\begin{table}
	\centering
	\begin{tabular}{c|c|c}
		Layer type & Layer output shape & Parameters
		\\
		\hline    															
		Normalisation & $28 \times 28$    &    0         
		\\																
		$5 \times 5$ Convolution       & $24 \times 24 \times 20$    &    520       
		\\																
		Max pooling & $12 \times 12 \times 20$   &    0         
		\\																
		$5 \times 5$ Convolution      &   $8 \times 8 \times 20$      &    10,020     
		\\																
		$3 \times 3$ Convolution     &   $6 \times 6 \times 40$      &    7,240      
		\\																
		Global average pooling  &  40           &    0                 
		\\																			
		Dense        &    50            &    1,230      
		\\																
		Dense       &     9            &    310     
	\end{tabular}
	\caption{The architecture of the neural network trained on the MNIST dataset, with a total of 20,340 parameters. All layers use a RELU activation function except the last which uses softmax to provide a class output. An $\ell^2$ regulariser is applied to the convolutional layers. The normalisation layer also incorporates random rotations and translations to aid training.}
	\label{table:NNparameters}
\end{table}

To simulate the few shot learning problem, we removed a digit from the training set, and trained the network from a random initialisation for 700 epochs on just the remaining 9 digits with a sparse categorical cross entropy loss.
The output of the network in the penultimate, 50 dimensional, dense layer was defined to be our feature space, and the network mapping raw images into this 50 dimensional space was regarded as the feature mapping $\phi$.
The feature vectors of 10 randomly selected examples of the missing digit were used as the training sample for the few shot learning algorithm.
The feature vectors of the old and new classes were normalised in feature space by translating the mean of the old training data to the origin and scaling by the maximum norm observed in the either training set.

In this setting we are able to consider the binary classification problem of recognising whether an image contains an example of the new digit, or of one of the 9 old digits.
To solve this problem, the empirical mean $\mu$ of the 10 examples of the additional digit was used to define the linear classifier~\eqref{eq:linearClassifier} in the feature space, where the centre $c_{Z}$ of the old data class was taken as the empirical mean of the (abundant) training data for the 9 previously learned digits.
For each digit, we computed ROC curves with respect to the threshold parameter $\theta$ of the classifier for the unseen test set of each of the various digit classes.
Calculating the area under this curve then provides a parameter-independent measure of the performance of the classifier.
This process was repeated using 20 different random samples of images from the extra class to train the classifier, and average values of the areas under the resulting ROC curves were computed.
The results are reported in Figure~\ref{fig:digit_aurocs}, corresponding to degree $k = 1$.

\begin{figure}
    \centering
	\begin{subfigure}{0.58\textwidth}
		\includegraphics[width=\textwidth]{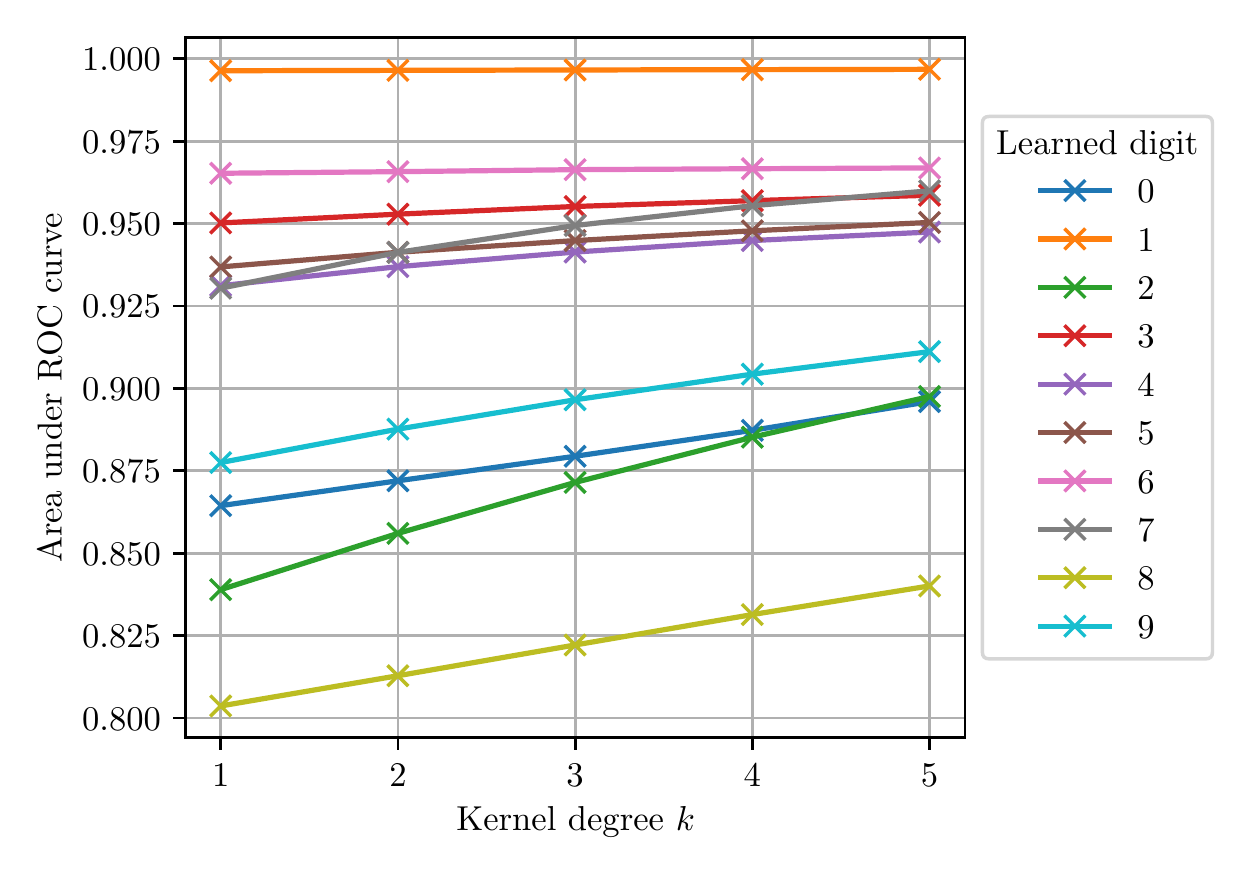}
		\caption{Average area under ROC curves.}
	\end{subfigure}%
	\begin{subfigure}{0.42\textwidth}
    	\includegraphics[width=\textwidth]{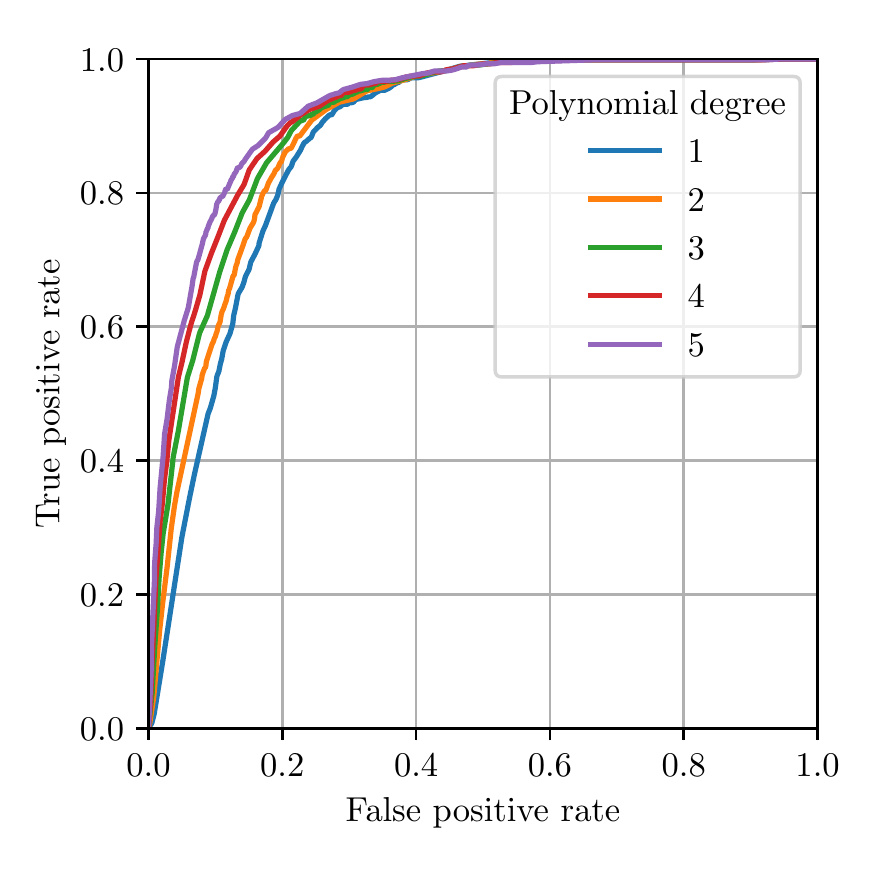}
		\caption{Sample ROC curves for learning digit 9.}
	\end{subfigure}
	\caption{ROC data obtained when learning each MNIST digit using the few-shot learning algorithm in a neural network feature space with various degree $k$ polynomial kernels. Degree $k = 1$ corresponds to learning directly in the network feature space. Left: Area under ROC curves, averaged over 20 few-shot training samples, each containing 10 examples. Right: ROC curves obtained for a single representative sample when learning the digit 9.}
	\label{fig:digit_aurocs}
\end{figure}

We repeated these experiments using polynomial kernels in the network feature space, to investigate whether this would further improve the performance of the classifier.
In this case, the complete feature map can be considered to be the composition of the network feature map, followed by the normalisation step outlined above, followed by the polynomial feature map.
The results from these experiments are also reported in Figure~\ref{fig:digit_aurocs}.
To make the results fully comparable, the same random training samples were used for the experiments with no additional kernel and with each of the polynomial kernels.

The results demonstrate that the proposed few shot learning algorithm works well, with a high probability of both successfully learning the new digit and of recalling the previous training to correctly classify the remaining digits.
Moreover, incorporating the polynomial kernel appears to make the two classes more easily separable and increasing the polynomial degree increases the area under the ROC curve, as predicted by the experiments on polynomial feature maps in the previous sections.

\section{Conclusion}\label{sec:conclusion}

This paper provides, for the first time, a comprehensive theoretical treatment of the challenging problem of learning from few examples.
The main thrust of our work is to investigate whether applying nonlinear feature mappings to the data can accelerate the onset of the blessing of dimensionality.
By explicitly incorporating these nonlinear mappings, widely used in practice but frequently neglected in theoretical work, we have been able to reveal their fundamental relationships with the geometry of the data distributions which, if preserved, can ensure successful learning without catastrophically forgetting previously learned tasks.
The consequences of these abstract results have been investigated in detail, both analytically and numerically, including in the feature space formed by neural networks trained on an image classification task.

Yet this remains just a beginning and many key open questions remain, which we plan to tackle in future work.
For instance, it could be beneficial to incorporate these principles into the design of future AI models in such a way as to make it easier to learn extra classes in future.
Similarly, it would be beneficial to develop tools to cheaply assess the suitability of a given AI model for learning new classes.

\section*{Acknowledgements}

The authors are grateful for financial support by  the UKRI and EPSRC (UKRI Turing AI Fellowship ARaISE EP/V025295/1). I.Y.T. is also grateful for support from the UKRI Trustworthy Autonomous Systems Node in Verifiability EP/V026801/1.

\bibliographystyle{plain}
\bibliography{references}

\appendix

\section{Notation}\label{sec:notation}
Throughout, we use the following notation:
\begin{itemize}
	\item {$\Real$ denotes the field of real numbers, $\Real_{\geq 0}=\{x\in\Real : x\geq 0\}$, and} $\Real^d$ stands for the $d$-dimensional linear real vector space;
	\item $\Natural$ denotes the set of natural numbers;
	\item bold symbols $\boldsymbol{x} =(x_{1},\dots,x_{n})$ will denote elements of $\Real^d$;
	\item $(\boldsymbol{x},\boldsymbol{y})=\sum_{k} x_{k} y_{k}$ is the inner product of $\boldsymbol{x}$ and $\boldsymbol{y}$, and $\|\boldsymbol{x}\|=\sqrt{(\boldsymbol{x},\boldsymbol{x})}$ is the standard Euclidean norm  in $\Real^d$;
	\item  $\mathbb{B}_d$ denotes the unit ball in $\Real^d$ centered at the origin:
	\[\mathbb{B}_d=\{\boldsymbol{x}\in\Real^d \suchthat {\|\boldsymbol{x}\|\leq 1}\};\]
	\item  $\mathbb{B}_d(r,\bfy)$  stands for the ball in $\Real^d$ of radius ${r> 0}$ centered at $\bfy$: 
	\[\mathbb{B}_d(r,\bfy)=\{\boldsymbol{x}\in\Real^d \suchthat {\|\boldsymbol{x}-\bfy\|\leq r}\};\]
	\item $V_d$ is the $d$-dimensional Lebesgue measure, and $V_d(\mathbb{B}_d)$ is the volume of unit $d$-ball;
	\item $\{x\}_{+}$ denotes the non-negative part of the argument $x$, given by $\{x\}_{+} = \max\{0, x\}$
\end{itemize}

\end{document}